\documentclass{article}

\usepackage{microtype}
\usepackage{graphicx}
\usepackage{subfigure}
\usepackage{booktabs} 
\usepackage{epsfig}
\usepackage{amsmath}
\usepackage{amssymb}
\usepackage{amsfonts}
\usepackage{caption}
\usepackage{amsthm}

\usepackage{url}
\usepackage{pifont}
\usepackage{dirtytalk}
\usepackage{bbm}
\usepackage{booktabs}
\usepackage{longtable}

\newcommand*{\thead}[1]{%
\multicolumn{1}{l}{\bfseries\begin{tabular}{@{}c@{}}#1\end{tabular}}}

\usepackage{bm}
\usepackage{nicefrac}
\usepackage{fancyhdr}
\usepackage{lastpage}
\usepackage{setspace}
\usepackage{soul}
\usepackage{xspace}

\usepackage{array,multirow}
\usepackage{mathtools}
\usepackage{adjustbox}
\usepackage{colortbl}
\usepackage{float,wrapfig}
\usepackage{xcolor}

\usepackage[english]{babel}
\newtheorem{theorem}{Theorem}

\newtheorem{lemma}[theorem]{Lemma}

\usepackage{todonotes}

\makeatletter
\if@todonotes@disabled

\else

\fi
\makeatother

\usepackage{mathtools}

\usepackage{enumitem}

\newlist{Properties}{enumerate}{2}
\setlist[Properties]{label=Property \arabic*.,itemindent=*}

\renewcommand{\vec}[1]{\boldsymbol{#1}}
\newcommand{\mat}[1]{\text{\bf #1}}

\DeclareMathOperator*{\argmax}{arg\,max}

\newcommand*\colourcheck[1]{%
  \expandafter\newcommand\csname #1check\endcsname{\textcolor{#1}{\ding{52}}}%
}
\colourcheck{blue}
\colourcheck{green}
\colourcheck{red}

\usepackage{hyperref}

\usepackage{amsthm}
\usepackage[utf8]{inputenc}
\usepackage[english]{babel}

\newcommand{\bracenom}{\genfrac{\lbrace}{\rbrace}{0pt}{}}

\usepackage{cleveref}

\usepackage[accepted]{icml2021}

\icmltitlerunning{Semantically Conditioned Negative Sampling for Efficient Contrastive Learning}

\begin{document}

\twocolumn[
\icmltitle{Semantically-Conditioned Negative Samples for Efficient Contrastive Learning}



\icmlsetsymbol{equal}{*}

\begin{icmlauthorlist}
\icmlauthor{James O' Neill}{uol}
\icmlauthor{Danushka Bollegala}{uol}

\end{icmlauthorlist}

\icmlaffiliation{uol}{Department of Computer Science, University of Liverpool, Liverpool, England}

\icmlcorrespondingauthor{James O' Neill}{james.o-neill@liverpool.ac.uk}
\icmlcorrespondingauthor{Danushka Bollegala}{danushka.bollegala@liverpool.ac.uk}

\icmlkeywords{Machine Learning, ICML}

\vskip 0.3in
]



\printAffiliationsAndNotice{} 

\begin{abstract}
Negative sampling is a limiting factor w.r.t. the generalization of metric-learned neural networks. We show that uniform negative sampling provides little information about the class boundaries and thus propose three novel techniques for efficient negative sampling: drawing negative samples from (1) the top-$k$ most semantically similar classes, (2) the top-$k$ most semantically similar samples and (3) interpolating between contrastive latent representations to create pseudo negatives. 
Our experiments on CIFAR-10, CIFAR-100 and Tiny-ImageNet-200 show that our proposed \textit{Semantically Conditioned Negative Sampling}
and Latent Mixup lead to consistent performance improvements. In the standard supervised learning setting, on average we increase test accuracy by 1.52\% percentage points on CIFAR-10 across various network architectures. In the knowledge distillation setting, (1) the performance of student networks increase by 4.56\% percentage points on Tiny-ImageNet-200 and 3.29\% on CIFAR-100 over student networks trained with no teacher and (2) 1.23\% and 1.72\% respectively over a \textit{hard-to-beat} baseline~\cite{hinton2015distilling}.




\end{abstract}

\section{Introduction}
Training deep neural networks using contrastive learning has shown state of the art (SoTA) performance in domains such as computer vision~\cite{oord2018representation,he2020momentum,chen2020improved,henaff2020data}, speech recognition~\cite{oord2018representation} and natural language processing~\cite{mueller2016siamese,logeswaran2018efficient,fang2020cert}. The generalization performance in contrastive learning heavily relies on the quality of negative samples used during training to define the classification boundary and requires a large number of negative samples. Hence, contrastive learning relies on techniques that enable training of large batch sizes, such as learning a lookup table~\cite{xiao2017joint,wu2018unsupervised,he2020momentum} to store lower-dimensional latent features of negative samples. However, training large models using contrastive learning can be inefficient when using uniformly sampled negatives (USNs) as the total number of potential negative sample pairs is $\mathcal{O}((N - N_+) N)$ where $N$ is the number of training samples and $N_+$ is the number of positive class samples. Hence, even a lookup table may poorly estimate negative sample latent features, even for a large number of USNs and training epochs.

A complementary approach to improve the learning efficiency of a contrative learned neural network is to reduce the model size using model compression techniques such as knowledge distillation~\citep[KD;][]{bucilua2006model}. In neural networks, this is achieved by transferring the logits of a larger ``teacher'' network to learn a smaller ``student'' network~\cite{hinton2015distilling}. There has been various KD methods proposed~\cite{romero2014fitnets,hinton2015distilling,zagoruyko2016paying,yim2017gift,passalis2018learning,tung2019similarity,peng2019correlation}. They involve minimizing KL divergence (KLD) between the student network and teacher network logits~\cite{hinton2015distilling}, minimizing the squared error between the student and teacher network intermediate layers~\cite{romero2014fitnets}, metric learning approaches ~\cite{tung2019similarity,ahn2019variational,tian2019contrastive,park2019relational}, attention transfer of convolution maps~\cite{zagoruyko2016paying} and activation boundary transfer~\cite{heo2019knowledge}.

In this paper, we propose three efficient \emph{negative sampling} (NS) alternatives to uniform NS and show their efficacy in standard supervised learning and the aforementioned KD setting. Our NS techniques are also complementary to the aforementioned lookup tables used for retrieving negative latent features. All three techniques have a common factor in that they produce negatives that are semantically similar to the positive targets on both instance- and class-levels. 
We collectively refer to these sampling methods as \emph{Semantically Conditioned Negative Sampling} (SCNS) as we replace a uniform prior over negative samples with a conditional probability distribution defined by the embeddings produced by a pretrained network. SCNS provides more informative negatives in contrast to USNs where many samples are \textit{easy to classify} and do not support the class boundaries. Additionally, it requires no additional parameters at training time and thus can be used with large training sets and models. The pretrained representations are used to estimate pairwise class-level and instance-level semantic similarities to define a top-$k$ NS probability distribution.
This reduces the number of negative pairs from $\mathcal{O}(N(N-N_+))$ to $\mathcal{O}(Nk)$ where $k$ is the number of nearest neighbor negative samples for each sample.
Below is a summary of our contributions.


\textbf{1) Class and Instance Conditioned Negative Sampling:} 
We define the NS distribution of each class by drawing negative samples proportional to the top-$k$ cosine similarity between pretrained word embeddings of the class labels. We also propose top-$k$ instance-level similarity for defining the NS distribution by performing a forward pass with a pretrained network prior to training.

\textbf{2) Contrastive Representation Mixup:} 
In \autoref{sec:latent_mixup}, we propose \emph{Latent Mixup} (LM), a variant of Mixup~\cite{zhang2017mixup} that operates on latent representations between teacher positive and negative representations to produce harder pseudo negative sample representations that lie closer to the class boundaries. This is also carried out for the student network representations and a distance (or divergence) is minimized between both mixed representations. 

\textbf{3) Theoretical Analysis of Conditioned Sampling:} We reformulate the mutual information lower bound to account for semantic similarity of contrastive pairs and describe the sample efficiency of SCNS compared to uniform sampling.  

\section{Related Research}\label{sec:rr}
Before describing our proposed methods, we review related work on the two most related aspects: efficient NS and KD. 

\textbf{Efficient Negative Sampling}
Efficient NS has been explored in the literature, predominantly for triplet learning. \textit{Semi-hard} NS has been used to sample negative pairs yet are still further in Euclidean distance than the anchor-positive pair~\cite{schroff2015facenet}.~\citet{oh2016deep} combine contrastive and triplet losses to mine negative structured embedding samples, drawing negative samples proportional to a Gaussian distribution of negative sample distances to the anchor sample~\citet{harwood2017smart}.~\citet{suh2019stochastic}
select hard negatives from the class-to-sample distances and then search on the instance-level within the selected class to retrieve negative samples.
~\citet{wu2017sampling} proposed a distance weighted sampling that selects more stable and informative samples when compared to uniform sampling and show that data selection is at least as important as the choice of loss function. 
~\citet{zhuang2019local} define two neighborhoods using $k$-means clustering, \textit{close} neighbors and dissimilar samples are \textit{background} neighbors.~\citet{wu2020mutual} use ball discrimination to discriminate between hard and easy negative unsupervised representations where positive pairs are different views of the same image. 
~\citet{tran2019improving} use the prior probability of observing a class to draw negative samples and then further sample instances within the chosen class based on the inner product with the anchor of the triplet, showing improvements over \textit{semi-hard} NS without the use of informative priors.

\textbf{Knowledge Distillation}
The original KD objective minimizes the Kullbeck-Leibler divergence between student network and teacher network logits~\cite{hinton2015distilling}.~\citet{romero2014fitnets} instead restrict the student network hidden representation to behave similarly to the teacher network hidden representations by minimizing the squared error between corresponding layers of the two networks.
The main restriction in this method is that both networks have to be the same depth and of similar architectures. 
Attention Transfer (AT)~\cite{zagoruyko2016paying} performs KD by forcing the student network to mimic the attention maps over convolutional layers of the pretrained teacher network. 
~\citet{passalis2018learning} use Gaussian and Cosine-based kernel density estimators (KDEs) to maximize the similarity between the student and teacher probability distributions. 
They find consistent improvements over \citet{hinton2015distilling} and \textit{Hint layers} used in Fitnet~\cite{romero2014fitnets}.
Similarity-Preserving (SP)~\cite{tung2019similarity} KD ensures that the activation patterns of the student network are similar to that in the teacher network for semantically similar input pairs. 
~\citet{peng2019correlation} propose Correlation Congruence (CC) to maximize multiple cross-correlations between samples of the same class. 
~\citet{ahn2019variational} provide an information-theoretic view of KD by maximizing the mutual information between student and teacher networks through variational information maximization~\cite{barber2003algorithm}. Moreover,~\citet{park2019relational} argue that the distance between relation structures created from multiple samples of the student and teacher networks should be minimized. 
They propose Relational KD (RKD), which involves the use of both distance-wise and angle-wise distillation losses that penalize structural discrepancies between multiple instance outputs from both networks. 
Contrastive Representation Distillation~\cite{tian2019contrastive} uses a CL objective to maximize a lower bound on the mutual to capture higher order dependencies between positive and negative samples, adapting their loss from~\citet{hjelm2018learning}.

\section{Methodology}
We begin by defining a dataset as $\mathcal{D} := \{(\vec{x}_i, \vec{y}_i)\}_{i = 1}^{N}$, 
which consists of $N$ samples of an input vector $\vec{x} \in \mathbb{R}^n$ and a corresponding target $\vec{y} \in \{0 ,1\}^C$ where sample $s_i:= (\vec{x}_i, \vec{y}_i)$ and $C$ is the number of classes. In the CL setting $\vec{x} = (\vec{x}_{*}, \vec{x}_{+} ,\vec{x}_{-,1},.., \vec{x}_{-, M})$ where $X_{+} := (\vec{x}_{*}, \vec{x}_{+})$ and $X_{-}:= (\vec{x}_{*}, \vec{x}_{-, 1}\ldots x_{-,M})$ for $M$ negative pairs. We denote a neural network as $f_{\theta}(\vec{x})$ which has parameters $\theta := (\theta_1, \theta_2,\ldots \theta_{\ell} \ldots, \theta_L)^T$ where $\theta_{l}:= \{\mat{W}_{l}, \vec{b}_{l}\}$, $\mat{W}_{l} \in \mathbb{R}^{d_l \times d_{l+1}}$, $\vec{b} \in \mathbb{R}^{d_{l+1}}$ and $d_l$ denotes the dimensionality of the $l$-th layer. 
The input to each subsequent layer is denoted as $\vec{h}_{l} \in \mathbb{R}^{d_{l}}$ where $\vec{x} \coloneqq \vec{h}_0$ and the corresponding output activation is denoted as $\vec{z}_l = g(\vec{h}_{l})$.
For brevity, we refer to $\vec{z} = g(\vec{h}_L)$ as the \emph{unnormalized output} where $g: \mathbb{R}^{d_L} \to \mathbb{R}^p$ and $\vec{z} \in \mathbb{R}^{p}$. However, when using a metric loss, $g: \mathbb{R}^{d_L} \to \mathbb{R}^{d_L}$ and therefore $\vec{z} \in \mathbb{R}^{d_L}$. In the former case, the cross-entropy loss is used for supervised learning and defined as $\ell_{\mathrm{CE}}(\mathcal{D}):= \frac{1}{N}\sum_{i=1}^N \sum_{c=1}^{p} -\vec{y}_{i, c}, \log \vec{\hat{y}}_{i, c}$ where $\vec{\hat{y}}_i = \sigma (f_\theta(\vec{x}_i)/\tau)$, $\vec{\hat{y}}_i \in \mathbb{R}^{p}$ and $\tau \in (0, + \infty)$ is the temperature of the softmax $\sigma$. 

 We also consider the KD setting where a student network $f_{\theta}^{\mathcal{S}}$ learns from a pretrained teacher network $f_{\omega}^{\mathcal{T}}$ with pretrained and frozen parameters $\omega$. 
The last hidden layer representation of $f_{\theta}^{\mathcal{S}}$ is given as $\vec{z}^{\mathcal{S}}:= f_{\theta}^{\mathcal{S}}(\vec{x})$ and similarly $\vec{z}^{\mathcal{T}} := f_{\omega}^{\mathcal{T}}(\vec{x})$. 
The Kullbeck-Leibler Divergence (KLD), $D_{\mathrm{KLD}}$, between $\vec{z}^{\mathcal{S}}$ and $\vec{z}^{\mathcal{T}}$ is defined in \autoref{eq:kld_div}

\begin{equation}\label{eq:kld_div}
\begin{split}
D_{\mathrm{KLD}}(\vec{y}^{\mathcal{T}} ||\vec{y}^{\mathcal{S}}) = \mathbb{H}( \vec{y}^{\mathcal{T}}) - \vec{y}^{\mathcal{T}} \log(\vec{y}^{\mathcal{S}}) \\
\vec{y}^{\mathcal{S}} = \sigma(\vec{z}^{\mathcal{S}}/\tau), \quad \vec{y}^{\mathcal{T}} = \sigma(\vec{z}^{\mathcal{T}}/\tau) 
\end{split}
\end{equation}

where $\mathbb{H}(\vec{y}^{\mathcal{T}})$ is the entropy of the teacher distribution $\vec{y}^{\mathcal{T}}$. 
Following~\citet{hinton2015distilling}, the weighted sum of cross-entropy loss and KLD loss shown in \autoref{eq:kld_loss} is used as our main KD baseline, where $\alpha \in [0, 1]$. 


\begin{equation}\label{eq:kld_loss}
\ell_{\mathrm{KLD}} = (1 - \alpha) \ell_{\mathrm{CE}}(\vec{y}^{\mathcal{S}}, \vec{y}) + \alpha \tau^2 D_{\mathrm{KLD}}\big(\vec{y}^{\mathcal{S}}, \vec{y}^{\mathcal{T}}\big) 
\end{equation}
To carry out KD using the KLD loss, the outputs of the pretrained teacher  $f^{\mathcal{T}}_{\omega}$ are stored after performing a single forward pass over mini-batches $\mathcal{B} \subset \mathcal{D}$ in our training set. These outputs are then retrieved for each mini-batch update of the smaller student network $f^{\mathcal{S}}_{\theta}$. Given this background, the next two subsections will describe our three main approaches to improving NS in contrastive learning.

\subsection{Semantically Conditioned Negative Sampling}\label{sec:scns}
Here, we describe two of our three approaches for improving NS efficiency that both involve using $f^{\mathcal{T}}_{\omega}$ to define a NS distribution. The first involves a cross-modal teacher network (i.e pretrained word embeddings for image classification) to define a \textbf{class-level} NS distribution and in the second $f_{\omega}^{\mathcal{T}}$ is a pretrained image classifier that defines an \textbf{instance-level} NS distribution. 

\subsubsection{Class-Level Negative Sampling}\label{sec:class_ns}
Our first method assumes that word embedding similarity between class labels highly correlates with image embedding similarity~\cite{leong2011measuring,frome2013devise}. Pretrained word embedding similarities are used to improve the sample efficiency of NS in contrastive learning and replace uniform NS that is typically used. The cosine similarity is measured between the pretrained word embeddings $(\vec{z}^{\mathcal{T}}_{w_i}, \vec{z}^{\mathcal{T}}_{w_j})$ where $(w_i, w_j)$ are the class labels in the vocabulary $\mathcal{V}$ and $|\mathcal{V}|=C$. This is carried out for all pairs to construct an all pair cosine similarity matrix $\mat{Z}_{\mathcal{V}} \in \mathbb{R}^{|\mathcal{V}|\times |\mathcal{V}|}$ that is then row-normalized with the softmax function $\sigma$ as $\mat{P}_{\mathcal{V}}:= \sigma(\mat{Z}_{\mathcal{V}}/\tau)$. Here, setting $\tau$ high leads to harder negative samples being chosen from the most similar classes. $\mat{P}$ represents the conditional probability matrix used to define $X_{-}$ by drawing samples as \autoref{eq:drawing_neg_samp} where $\mathcal{D}_w$

\begin{equation}\label{eq:drawing_neg_samp}
    x_{-} \sim \mathcal{D}_w \propto \mat{P}_{w}
\end{equation}

 represents all samples $(x^1_w, \ldots x^M_w)$ for a given class associated with $w$. This is repeated $M$ times when using CL.
\paragraph{Hard $k$-Nearest Class-Level Negative Samples}
Instead of sampling over a possible $M=|\mathcal{V}|-1$ number of negative samples, we can define the top-$k$ most similar \emph{hard} negative samples. The top-$k$ cosine similarities from other labels in $\mathcal{V}$ are selected by applying \autoref{eq:class_cosine} where $z^{\mathcal{T}}_{w_i} \in \mathbb{R}^{d_w}$ and $\vec{z}_{w_i} := f^{\mathcal{T}}_{\theta}(w_i)$ of the class label $w_i$. 

\begin{equation}\label{eq:class_cosine}
\begin{split}
\mathrm{topk}_w(\vec{z}^{\mathcal{T}}_{w_i}) = \argmax_{k \neq i}\big[\cos(\vec{z}^{\mathcal{T}}_{w_i},\vec{z}^{\mathcal{T}}_{w_k })\big]
\end{split}
\end{equation}

The $k$-nearest neighbor ($k$-NN) similarity scores are then stored in $\mat{Z}_{\mathcal{V}}^{k} \in \mathbb{R}^{|\mathcal{V}|\times k}$ with a corresponding matrix  $I_{\mathcal{V}}^{k} \in \mathbb{R}^{|\mathcal{V}|\times k}$ that stores the $k$-NN class indices $\mat{Z}_{\mathcal{V}}^{k}$ retrieved by applying \autoref{eq:class_cosine}. 
We focus on only sampling from the top-$k$ classes and therefore define the normalized top-$k$ class distribution matrix as $\mat{P}^{k}_{\mathcal{V}} := \sigma(\mat{Z}^{k}_{\mathcal{V}}/\tau)$. A row vector of $\mat{P}^{k}_\mathcal{V}$ is denoted as $\vec{P}^{k}_w$ consisting of the $k$ truncated conditional probabilities corresponding to the $k$ nearest class labels of $w \in \mathcal{V}$. We then sample as in \autoref{eq:drawing_neg_samp} instead with top-$k$ negative samples $\mathcal{D}^{k}_w \subset  \mathcal{D}_w$ and $|\mathcal{D}^{k}_w| = k |\mathcal{D}_w|$.

\begin{figure}[ht]
    \centering
  \includegraphics[width=.95\linewidth]{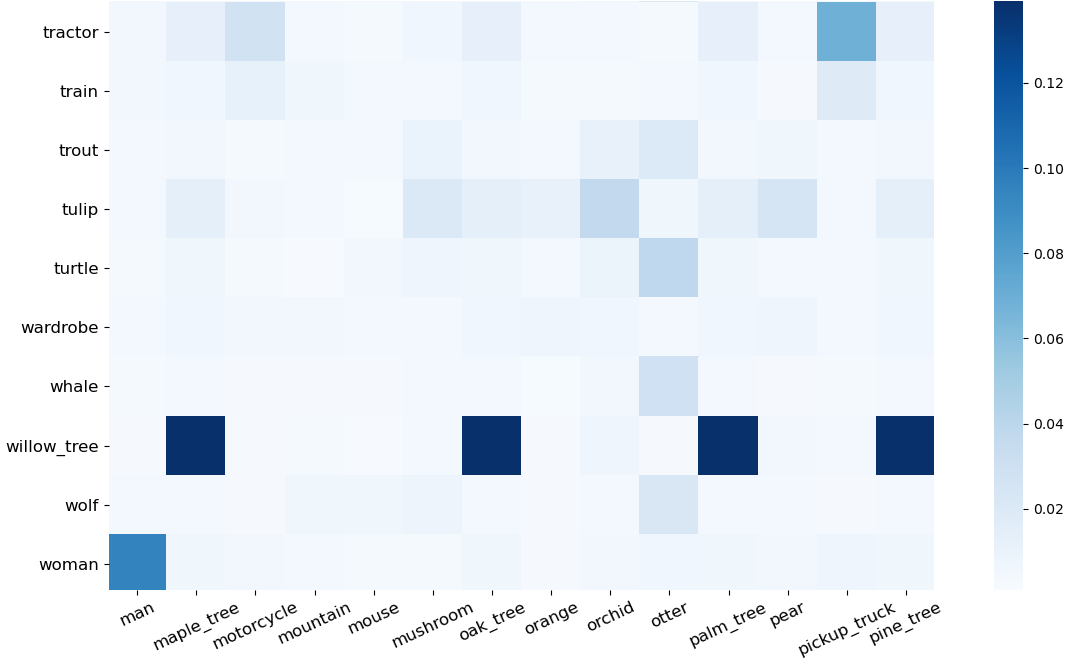}
    \caption{CIFAR-100 subset of word embedding class similarities}
    \label{fig:semantic_classes}
\end{figure}

Figure \eqref{fig:semantic_classes} shows a submatrix of $\mat{P}_{\mathcal{V}}$ as a heatmap corresponding to a subset of  CIFAR-100 class labels on the x and y-axis. We see that ``willow-tree'' has a high similarity score with ``maple-tree", ``oak-tree", ``palm-tree" and ``pine-tree". Therefore, samples from these class labels will be sampled more frequently as hard yet more informative negative samples. Similarly, (``man", ``woman") and (``tractor", ``pickup-truck") would be sampled at a higher rate than the remaining terms.




\subsubsection{Instance-level Conditioned Sampling}

Class-level SCNS (or Class-SCNS) may not be granular enough as the conditional probability assigned to a class is the same for all samples within that class. In instance-level SCNS we define the top-$k$ nearest \emph{samples} for each $\vec{x} \in \mathcal{D}$. A top-$k$ instance similarity matrix produced by $f_{\omega}^{\mathcal{T}}$ is iteratively constructed  $\forall \vec{x} \in \mathcal{D}$
and the outputs are stored in $\mat{Z}^{k}_{\vec{x}} \in \mathbb{R}^{N \times k}$. As before, we define $\mat{P}^{k}_{\vec{x}}$ and sample $\vec{x}_{-} \sim \mathcal{D}^{k}_{x} \propto \mat{P}^{k}_{\vec{x}}$. Unlike Class-SCNS, $f^{\mathcal{T}}_{\omega}$ is trained on the same modality (e.g images) as $f^{\mathcal{S}}_{\theta}$ and $\mat{P}_{x}^{k}$ is now a SCNS matrix for each $\vec{x}_{*} \in \mathcal{D}$ and not per class label $w \in \mathcal{V}$. 

For image classification, we choose $f_{\omega}^{\mathcal{T}}$ to be a pretrained CNN. We use the final hidden representation $\vec{z}_{x_{*}}^{\mathcal{T}} \in \mathbb{R}^{d}$ which is a latent representation of an input image $\mat{x} \in \mathbb{R}^{C_{i} \times d_{w_i} \times d_{h_i}}$ where $C_i$ is the number of input channels, $d_{w_i}$ is the width of the input image and $d_{h_i}$ is the input image height. However, if the $l$-th intermediate layer activation $\mat{h} \in \mathbb{R}^{C_{o}\times d_{h_o}\times d_{h_o}}$ is used to measure semantic similarity, we vectorize $\mat{h}$ to $\vec{h} \in \mathbb{R}^{C_{o}d_{h_o} d_{h_o}}$, where $C_o$ is the number of output channels, $d_{h_o}$ is the height of the output feature maps and $d_{w_o}$ is the output width. We note that Instance-SCNS may be favoured over Class-SCNS particularly when $|\mathcal{V}|$ is relatively low and $k \approx |\mathcal{V}|$ (e.g CIFAR-10). 

\subsubsection{Conditioned Samples with a Lookup Table}
To further reduce training time we can combine SCNS with a lookup table~\cite{xiao2017joint} that stores negative sample features and updates during training. 
%
For $\forall X_{+} \in \mathcal{B}$, the dot product is computed between $\vec{z}^{\mathcal{S}}$ and the $i$-th column of the $L_2$ row normalized lookup table $\mat{V} \in \mathbb{R}^{d_l \times N_v}$. If the $i$-th target $y_i$ is predicted then the update $\vec{v}_i \leftarrow \gamma \vec{v}_i + (1 - \gamma)\vec{z}^{\mathcal{S}}$ is performed in the backward pass. 
A high $\gamma \in [0, 1]$ results in a smaller update $\vec{v}_i \in \mathbb{R}^{d_l}$ in the lookup table. 
For $X_{-}$, a circular queue $\mat{Q} \in \mathbb{R}^{d_l \times N_q}$ is used where $N_q$ is the queue size and the dot product $\mat{Q}^{\top}\vec{z}^{\mathcal{S}}$ is computed. The current features are put into the queue and older features are removed from the queue during training. 
\autoref{eq:queue_v} shows the conditional probability $p_i$ corresponding to the target $\vec{y}_i$ where $\vec{u}_{n} \in \mathbb{R}^{d_l}$ is stored as the $n$-th row of $\mat{Q}$. . 

\begin{equation}\label{eq:queue_v}
p_i = \frac{e^{(\vec{v}^{\top} \vec{z}^{\mathcal{S}}_{*}/\tau )}}{ \sum_{m=1}^{N_v}e^{(\vec{v}^{\top}_m \vec{z}^{\mathcal{S}}_{*}/\tau )} 
 + \sum_{n=1}^{N_q} e^{(\vec{u}^{\top}_n \vec{z}^{\mathcal{S}}_{*}/\tau )}}  
\end{equation}

For the anchor sample $\vec{x}_{*}$, the conditional probability that the representation of a negative sample $\vec{x}_{-}$ matches in the circular queue is given by \autoref{eq:queue_q}.

\begin{equation}\label{eq:queue_q}
q_i = \frac{e^{(\vec{u}^{\top}_i \vec{z}^{\mathcal{S}}_{*}/\tau )}}{\sum_{m=1}^{N_v} e^{(\vec{v}^{\top}_m \vec{z}^{\mathcal{S}}_{*}/\tau )} 
+ \sum_{n=1}^{N_q} e^{(\vec{u}^{\top}_n \vec{z}^{\mathcal{S}}_{*}/\tau )}}
\end{equation}

The gradient $\nabla_{z_{\mathcal{S}}}\mathbb{E}_{\vec{z}^{\mathcal{S}}}[\log p_i]$ is defined in the backward pass as \autoref{eq:bp_queue},  

\begin{equation}\label{eq:bp_queue}
\frac{\partial \ell}{\partial \vec{z}^{\mathcal{S}}} = \frac{1}{\tau} \Big[(1 - p_i)\vec{v} - \sum_{\substack{j=1, \\ j \neq y}}^{N_V} p_j \vec{v}_j - \sum_{k=1}^{N_q} q_k \vec{u}_k \Big]
\end{equation}

where $y$ is the column of $\mat{V}$ corresponding to the $y$-th target $y \in \mathbb{N}_{+}$ and $\vec{y}^{\mathcal{T}}$ in the KD setting. This lookup table can be used complementary to SCNS and we use it in our experiments. It is also easier to compare to prior CL approaches~\cite{tian2019contrastive} as they too use a lookup table.

\begin{algorithm}[t]
\caption{\label{alg:scns} SCNS distillation algorithm.}
\begin{algorithmic}[1]
    \STATE \textbf{input:} \small{mini-batch size $M$, number of batches $N$, student network $f_{\theta}$, teacher network $g_{\varphi}$, regularization terms $\gamma_{+}, \gamma_{-}$.}
    
    \FOR{$k$-NN sampled minibatch $\{\bm B_i\}_{i=1}^N$}
    \STATE $\ell = 0$
    \STATE \textbf{for all} $\{\bm x\}_{j=1}^M$ \textbf{do}

    \STATE $~~~~$ \textcolor{gray}{\# Positive embedding features}

   \STATE $~~~~~$$\bm h^{\mathcal{S}}_{*}, \bm h^{\mathcal{S}}_{+}, \bm h^{\mathcal{S}}_{-}  = f^{\mathcal{S}}(\bm x_{*}), f^{\mathcal{S}}(\bm x_{+}), f^{\mathcal{S}}(\bm x_{-})$ 
   
   \STATE $~~~~~$$\bm h^{\mathcal{T}}_{*}, \bm h^{\mathcal{T}}_{+}, \bm h^{\mathcal{T}}_{-} = f^{\mathcal{S}}(\bm x_{*}), f^{\mathcal{T}}(\bm x_{+}), f^{\mathcal{T}}(\bm x_{-})$ 
  
      \STATE $~~~~$ \textcolor{gray}{\# contrastive features retrieved from lookup table}
    \STATE $~~~~~$$\bm z^{\mathcal{S}}_{*},  z^{\mathcal{S}}_{+},  z^{\mathcal{S}}_{-} = g^{\mathcal{S}}(\bm h_{*}), g^{\mathcal{T}}(\bm h_{+}), g^{\mathcal{T}}(\bm h_{-})$   

    \STATE $~~~~~$$\bm z_{*}^{\mathcal{T}}, z_{+}^{\mathcal{T}}, z_{-}^{\mathcal{T}} = g^{\mathcal{T}}(\bm h_{*}), g^{\mathcal{T}}(\bm h_{+}),  g^{\mathcal{T}}(\bm h_{-})$  
       
   \STATE $~~~$ \textcolor{gray}{\# Contrastive mixup representations}
   \STATE $~~~~~$$\bm \tilde{z}^{\mathcal{S}}, \tilde{z}^{\mathcal{T}} = \kappa(\bm z_{-}^{\mathcal{S}}, \bm z_{*}^{\mathcal{S}}), \kappa(\bm z_{-}^{\mathcal{T}}, \bm z_{*}^{\mathcal{T}}) $

    \STATE $~~~~$ \textcolor{gray}{\# student network prediction}
    \STATE $~~~~~$ $\vec{y}^{\mathcal{S}} =  \sigma(\bm h_{*}^{\mathcal{S}} \mat{W}^T)$
    
    \STATE \textcolor{gray}{~~~~~\# Cross-entropy loss}
    \STATE $~~~~$ $\ell := \ell + (1 - \alpha)\ell_{\mathrm{CE}}(\bm y^S, \bm y)$
    
    \STATE \textcolor{gray}{~~~~\# Latent Mixup Loss}
    \STATE $~~~$ $\ell := \ell + \alpha \ell_{\mathrm{KLD}}(\vec{\tilde{z}}^S, \vec{\tilde{z}}^T)$
    
    \STATE \textcolor{gray}{~~~~\# KD loss of positive and negative samples}
    \STATE $~~~$ $\ell := \ell - \gamma_{+} \ell_{\mathrm{KD}}(\vec{z}^{\mathcal{S}}_{+}, \vec{z}^{\mathcal{T}}_{+})$
    
    \STATE $~~~$ $\ell := \ell - \gamma_{-} \ell_{\mathrm{KD}}(\vec{z}^{\mathcal{S}}_{-},\vec{z}^{\mathcal{T}}_{-})$
    
    \STATE \textbf{end for}

    \STATE perform gradient updates on $f_\theta$ to minimize $\ell$
    \ENDFOR
    \STATE \textbf{return} encoder network $f(\cdot)$, and throw away $g(\cdot)$
\end{algorithmic}
\end{algorithm}

\subsection{Interpolating Contrastive Representations}\label{sec:latent_mixup}
Instead of using a pretrained network to define the negative samples that are close to the classification boundary, we can instead mix positive and negative representation to produce pseudo negative samples that are close to the positive sample. Mixup~\cite{zhang2017mixup} is a simple regularization technique that performs a linear interpolation of inputs.

Our proposed LM instead mixes the latent reprepsentations between positive and negative pairs given by the student and teacher networks as opposed to mixing the raw images. The motivations for this is that 
$f^{\mathcal{S}}_{\theta}$ learns more about the geometry of the embedding space induced by $f^{\mathcal{T}}_{\omega}$ and interpolating on a lower-dimensional manifold than the original input can lead to smoother interpolations. The interpolation function $\kappa(\vec{z}_i, \vec{z}_j)$ in~\autoref{eq:mixup} outputs a contrastive mixture $\tilde{\vec{z}}$ from $\vec{z}_i \in \mathbb{R}^{d}$, $\vec{z}_j \in \mathbb{R}^d$ and the mixture coefficient $\nu \in [0, 1]$ is drawn from the beta distribution $\nu \sim \mathrm{Beta}(\beta, \beta)$ where  $\beta \in [0, \infty]$ and $\beta \to 0$ approaches the empirical risk.
\begin{equation}\label{eq:mixup}
   \kappa(\vec{z}_i, \vec{z}_j) = \nu \vec{z}_i + (1 - \nu) \vec{z}_j
\end{equation}

Both student and teacher LM representations and teacher targets are then computed as,

\begin{gather}\label{eq:latent_mix}
\tilde{\vec{z}}^{\mathcal{S}}_{ij} = \kappa(\vec{z}^{\mathcal{S}}_i, \vec{z}^{\mathcal{S}}_j),  \quad
\tilde{\vec{z}}^{\mathcal{T}}_{ij} = \kappa(\vec{z}^{\mathcal{T}}_i, \vec{z}^{\mathcal{T}}_j) \\
\tilde{\vec{y}}_{ij}^{\mathcal{S}} = \sigma\big(\mat{W}^T \tilde{\vec{z}}^{\mathcal{S}}_{ij}/\tau\big) , \quad \tilde{\vec{y}}_{ij}^{\mathcal{T}} = \sigma\big(\kappa(\vec{y}_{i}^{\mathcal{T}}, \vec{y}_{j}^{\mathcal{T}})/\tau\big)
\end{gather}

where $\tilde{\vec{y}}_{ij}^{\mathcal{T}}$ is a synthetic bimodal mixup target. Henceforth, we will denote mixup teacher targets as $\tilde{\vec{y}}^{\mathcal{T}}$ and LM representations as $\tilde{z}^{\mathcal{S}}$ and $\tilde{z}^{\mathcal{T}}$. The objective can then be described by the KLD as \autoref{eq:kld} where $\mathbb{H}$ is the entropy of the predicted teacher distribution over classes $\tilde{\vec{y}}^{\mathcal{T}}$. When training from scratch with standard cross-entropy, the targets are mixed and renormalized with $\sigma$ where $\tau$ performs label smoothing resulting in a peaked bimodal distribution.  

\begin{equation}\label{eq:kld}
    D_{\mathrm{KLD}}(\tilde{\vec{y}}^{\mathcal{T}} || \tilde{\vec{y}}^{\mathcal{S}}) = \mathbb{H}(\tilde{\vec{y}}^{\mathcal{T}}) - \tilde{\vec{y}}^{\mathcal{T}} \log(\tilde{\vec{y}}^{\mathcal{S}})
\end{equation}

Instead of using contrastive representation mixup with the KLD distillation objective, we also use it to mix between latent representations in the CL setting whereby representations of negative and positive samples are mixed to produce pseudo-\textit{hard} negative sample representations. In this case $\tilde{\vec{z}}^{\mathcal{S}}:= \kappa(\vec{z}^{\mathcal{S}}_i, \vec{z}^{\mathcal{T}}_j)$ and similarly for the teacher network as shown in Line 12 of Algorithm \ref{alg:scns}.

\subsection{Theoretical Analysis of Conditioned Sampling}
In this subsection we reformulate the MI lower bound to include the notion of semantic similarity between negative samples and their corresponding anchor. We then describe the difference in sample complexity between USNs and SCNS w.r.t. observing the top-$k$ negative samples in the training data. We use the InfoNCE loss~\cite{oord2018representation} with SCNS for our experiments, as shown in~\autoref{eq:infonce_loss},

\begin{equation}\label{eq:infonce_loss}
    \ell = \underset{\substack{(\vec{x}_{*}, \vec{x}_{+}) \ \sim \ \mathcal{D}_{+} \\ \vec{x}_{-} \sim \mathcal{D}_{-} \propto \ P_{x}}}{\mathbb{E}} \Bigg[- \log \Bigg( \frac{\exp(\vec{z}^{\top}_{*} \vec{z}_{+})}{\exp(\vec{z}^{\top}_{*} \vec{z}_{+}) + \exp(\vec{z}^{\top}_{*} \vec{z}_{-})} \Bigg) \Bigg]
\end{equation}

where $\vec{x}_{-}$ is conditioned on the distribution $P_x$ as described in \autoref{sec:scns}. By minimizing the InfoNCE loss $\ell$, we maximize the MI between the positive pair $(\vec{x}_{*}, \vec{x}_{+})$. The optimal score for $f(\vec{x}_{*}, \vec{x}_{+})$ is given by
$p(x_{*}|x_{+})/p(x_{*})$, substituting this into ~\autoref{eq:infonce_loss} and splitting $\vec{x}$ into positive and negative samples $X_{-}$ gives:

\begin{equation}
\begin{gathered}
    \ell = - \underset{X}{\mathbb{E}} \log \Bigg[ \frac{p(x_{*}|x_i)/p(x_{*})}{\frac{p(x_{*}|x_{+})}{p(x_{*})} + \sum_{x_i \in X_{-}}\frac{p(x_i|x_{+})}{p(x_i)}} \Bigg] \\ 
    =\underset{X}{\mathbb{E}}\log \Bigg[1 + \frac{p(x_{*})}{p(x_{*}|x_{+})} + \sum_{x_i \in X_{-}} \frac{p(x_i|x_{+})}{p(x_i)}\Bigg]\\
    \approx \underset{X}{\mathbb{E}}\log \Bigg[1 + \frac{p(x_{*})}{p(x_{*}|x_{+})} (M-1) \underset{x_i}{\mathbb{E}}\frac{p(x_i|x_{+})}{p(x_i)}\Bigg] \\
    =  \underset{X}{\mathbb{E}}\log \Big[1 + \frac{p(x_{*})}{p(x_{*}|x_{+})} (M-1) \Big]
    \geq \underset{X}{\mathbb{E}}\log \Big[\frac{p(x_{*})}{p(x_{*}|x_{+})} M\Big] \\
    = - I(x_{*}, x_{+}) + \log(M) \label{eq:mi_lower_bound}
\end{gathered}
\end{equation}

From \autoref{eq:mi_lower_bound}, we see that $I(x_{*}, x_{+}) \geq \log(M) - \ell$ ~\cite{oord2018representation} and the larger the number of negative samples, $M$, the tighter the MI bound. 
However, we argue that if a pretrained $f^{\mathcal{T}}_{\omega}$ has training error close to $0$, then $\log(M)$ should be replaced with a term that accounts for the geometry of the embedding space as not all negative samples are equally important for reducing $\ell$. Therefore, we express how top-$k$ samples from SCNS tightens the lower bound estimate on MI when compared to using USNs. Given that we are not restricted to a distance, divergence or similarity between vectors, we refer to a general \textit{alignment} function $A$ that outputs an alignment score $a \in [0, 1]$. 

Given $\vec{x}_{*}$, the expected alignment score for the top-$k$ negative samples is $a^{k}_{x_{-}} := \mathbb{E}_{x_{-} \sim D^{k}_x} [A(z_{*}, z_{-})]$ and for negative samples outside of the top-$k$ samples, $a^{r}_{x_{-}} := \mathbb{E}_{x_{-} \sim D^{r}_x } [A(z_{*}, z_{-})]$ where $ D^{r}_{x_{*}} \subseteq D, D^{k}_{x_{*}}  \not\in D^r_{x_{*}}$, $r = N - N_y - k$. and $N_y$ is the number of samples of class $y$. 
 The alignment weight (AW) $\Omega_x:= 1 - a^{k}_{x}/(a^{k}_{x} + a^{r}_{x})$ is then used to represent the difference in `closeness' between the top-$k$ negative samples and the remaining negative samples. 
\begin{lemma}\label{prop_2}
Given $\Omega := \sum_{i=1}^M \Omega_{x_{*}}$, we can reformulate the MI lower bound when using SCNS as \autoref{eq:new_mi_lb}.
\begin{equation}\label{eq:new_mi_lb}
     I(X, Y) \geq \ell + \log(2\Omega)
\end{equation} 
\end{lemma}
\begin{proof}
We substitute $\log (2\Omega)$ for $\log(M)$ in \autoref{eq:mi_lower_bound} as $a^{k}_{x} \approx a^{r}_x$ in uniform sampling as $M \to \infty$. 
\end{proof}
This lower bound favors top-$k$ negative samples that have alignment with the positive class boundary and are relatively close compared to the negative samples outside of the top-$k$. It is dependent on the loss of $f^{*}$ where $\ell \approx 0$ results in an accurate alignment estimation for all embedding pairs. 
\begin{lemma}\label{prop_3}
In the worst case, when all $M$ negatives are equidistant to $x_{*}$, forming a ring on the $L_2$ embedding hypersphere, SCNS is equivalent to uniform sampling. 
\end{lemma}
\vspace{-.25em}
\begin{proof}
This holds as $\log(2\Omega) \approx \log(M)$ when the centroids of both sets $a^{k}_x = a^{r}_x$ and $\Omega = M/2$. Therefore, in the worst case SCNS is equivalent to uniform sampling. 
\end{proof}
\vspace{-.25em}
The above case can be due to degenerative representations used in the top-$k$ SCNS similarity computation (i.e $\epsilon_{tr}$ not close to 0) or a characteristic of the training data itself. Then, the relation between $M$ USNs and top-$k$ SCNSs can be formed as follows. Let $\Omega_{D_{x}^{\mathcal{U}}}$ be the AW of USNs for $x$ and $\Omega_{D^{k}_{x}}$ be the AW for the top-$k$ negative samples. When $|D_{x}^{\mathcal{U}}| \approx |D^{k}_{x}|$ we have,

\begin{equation}\label{eq:scns_milb}
    I(X, Y) \geq  \ell + \log(2\Omega_{D^k_x}) \geq \ell + \log(2\Omega_{\mathcal{U}})
\end{equation}

given the non-uniform prior over negative samples as defined in SCNS. For some $k \ll N$, $2(\Omega_{k} - \Omega_{\mathcal{U}})= 0$ is met when a subset of $D^{\mathcal{U}}_{x}$ negative samples have $\bar{z}^{\mathcal{U}}_{\tilde{x}} \approx \bar{z}^{k}_{x}$ where $\tilde{x}$ denotes the aforementioned subset and $\bar{z}^{k}_{x}$ is the centroid of the top-$k$ negative samples.

\paragraph{Number of Uniform Draws To Observe Top-$k$ SCNS}\label{sec:ccp_single}

We can now describe the expected number of USNs draws required to observe the top-$k$ samples at least once for a given $x_{*}$. Let $C_i$ denote the number of negative samples observed until the $i$-th new negative sample among the top-k samples is observed and $N$ is the total number of samples until all top-$k$ negative samples are observed.
Since $C=\sum^{k}_{i=1} C_i$, $\mathbb{E}[C] = \mathbb{E}\big[\sum_{i=1}^{k} C_i\big] = \sum_{i=1}^{k}\mathbb{E}\big[C_i\big]$
where $C_i$ follows a geometric distribution with parameter $(k+1-i)/M$. 
Therefore, $\mathbb{E}[C_i] = M/(k + 1 - i)$ and thus the expected number of draws is given by \autoref{eq:ccp}.
\begin{equation}\label{eq:ccp}
\mathbb{E}[C]= M \sum_{i=1}^k (k+1-i)^{-1} = M\sum_{i=1}^k i^{-1}
\end{equation}
\begin{figure*}[ht]
\centering     
\includegraphics[scale=0.5]{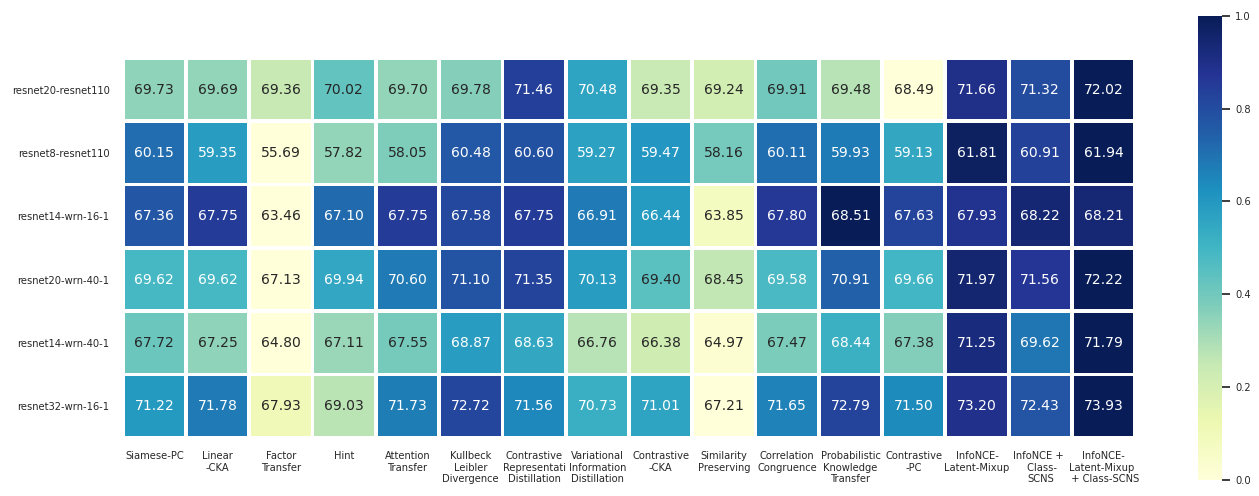}
\caption{CIFAR-100 Test Accuracy for Knowledge Distillation Approaches:\small{The y-axis model naming convention is $\langle$student-\# convolutional layers-teacher-\# convolutional layers-\# fully-connected -layers$\rangle$} and the x-axis denotes the KD method.}\label{fig:cifar100_heatmap}
\end{figure*}
We reformulate this for mini-batch training where consecutive batches of size $b$ for $x_{*}$ are drawn with replacement. This is a special case of the \emph{Coupon Collector's Problem}~\citep[(CCP)][]{von1954coupon}.

\begin{theorem}
 The batch variant of the CCP formulates the probability of the expected number of batches of size $b$ to observe top-$k$ SCNS samples at least once as \autoref{eq:bccp} 
\begin{gather}\label{eq:bccp}
    \sum_{j=0}^{\infty} P(K > ib) = \sum_{j=0}^{M-1}(\text{-}1)^{M-j+1}\binom{M}{j}\frac{1}{1 - (\frac{j}{M})^{b}}
\end{gather}
\end{theorem}
\begin{proof}
See \autoref{sec:scns_complexity} for the proof. 
\end{proof}
As $M$ grows more mini-batches are required to observe the top-$k$ \textit{hard} negative samples. Thus, $b$ has to be larger for uniform sampling to cover the informative negative samples, coinciding with the MI formulation in \autoref{eq:scns_milb}. Further justifications are found in  \autoref{sec:mi_and_scns} and \autoref{sec:scns_complexity}. 
\section{Experiments}\label{sec:experiments}
We now discuss the experimental results and note that additional details on hardware, datasets, model architectures and settings are found in \autoref{sec:hardware}, \ref{sec:architectures} and \ref{sec:experiment_details}. 

\paragraph{Standard Supervised Learning}
We first test our three NS approaches in the standard supervised learning setting on CIFAR-10. Several ResNet-based architectures~\cite{he2016deep,zagoruyko2016wide,xie2017aggregated} are trained with (1) cross entropy between LM representations and the cross-entropy between student predictions and targets (Cross-Entropy + LM) and CL (InfoNCE + LM), (2) only using CL with the LM representations of each contrastive pair (InfoNCE-LM) and (3) Class-SCNS (InfoNCE + Class-SCNS) and Instance-SCNS (InfoNCE + Instance-SCNS) with the InfoNCE loss.

\begin{table}[ht]
\centering
\resizebox{.48\textwidth}{!}{%
\begin{tabular}{l|llllll}
\toprule

Methods  & resnet20 & resnet32 & resnet110 & wrn-16-1 & wrn-40-1 & resnext32x4 \\

\midrule
\midrule
Cross-Entropy & 91.14 & 92.49 & 93.38 & 94.06 & 94.47 & 95.38 \\
\midrule
Cross-Entropy + LM ($\beta$=$0.5$) & 92.74 & 92.73 & 93.53 & 94.26 & 94.54 & 95.47 \\

\midrule
InfoNCE & 91.68 & 92.90 & 94.01 & 94.39 & 95.23 & 95.77 \\

\midrule
-LM ($\beta$=$0.01$) & 91.72 & 92.93 & 94.09 & 94.48 & 95.55 & 95.87  \\
-LM ($\beta$=$0.05$) & 91.90 & 93.08 & 94.24 & 94.67 & 94.81 & 95.89 \\
-LM ($\beta$=$0.1$) & 91.97 & 93.17 & 94.42 & 94.81 & 94.95 & 96.01  \\
-LM ($\beta$=$0.2$)  & 92.17 & 93.12 & 94.63 & 94.99 & 94.93 & 96.07 \\
-LM ($\beta$= $0.5$)  & \textbf{92.38} & \textbf{93.21} & \textbf{94.85} & \textbf{95.11} & \textbf{95.09} & \textbf{96.22} \\

\midrule
 + LM ($\beta$=$0.01$) & 91.92 & 93.12 & 94.17 & 94.53 & 95.68 & 95.97 \\
 + LM ($\beta$=$0.05$) & 91.98 & 93.20 & 94.53 & 94.95 & 94.73 & \textbf{96.94} \\
 + LM ($\beta$=$0.1$) & 92.04 & 93.49 & 94.44 & 94.79 & 94.99 & 96.24 \\
 + LM ($\beta$=$0.2$)  & 92.40 & 93.38 & 94.92 & 95.07 & 95.05 &  96.11 \\
 + LM ($\beta$=$0.5$)  & \textbf{92.96}$^{\dag\dag}$ & 93.56  & \textbf{95.02} & \textbf{95.29}$^{\dag\dag}$ & \textbf{95.13} & 96.08\\

\midrule
+ Class-SCNS ($k = 1$) & 92.04 & 93.08 & 94.61 & 93.97 & 95.52 & 95.82 \\
+ Class-SCNS ($k = 2$) & \textbf{92.44} & 93.10 & 94.83 & 94.42 & \textbf{96.03}$^{\dag\dag}$ & \textbf{96.09} \\
+ Class-SCNS ($k = 5$) & 91.98 & 93.02 & \textbf{94.43} & \textbf{93.88} & 95.23 & 95.77 \\

\midrule
+ Instance-SCNS ($k$=$|D|/5$) & 92.01 & 93.27 & 94.34 & 93.49 & 95.28 & 95.19 \\
+ Instance-SCNS ($k$=$|D|/10$) & 92.12 & 93.29 & 94.09 & 93.58 & 95.32 & 95.41 \\
+ Instance-SCNS ($k$=$|D|/20$) & 92.20 & 93.25 & 94.83 & 94.14 & 95.66 & 95.99 \\
+ Instance-SCNS ($k$=$|D|/100$) & \textbf{92.42} & \textbf{94.39} $^{\dag\dag}$ & \textbf{95.11}$^{\dag\dag}$ & \textbf{94.56} & \textbf{95.91} & 96.09 \\
-Instance SCNS ($k$=$|D|/500$) & 92.38 & 93.57 & 95.02 & 93.36 & 95.72 & \textbf{96.27}$^{\dag\dag}$\\
\bottomrule
\end{tabular}%
}
\caption{
\small{
Test \emph{accuracy} (\%) of student networks on CIFAR-10.
}}
\label{tbl:cifar10_ablation_results}
\end{table}


~\autoref{tbl:cifar10_ablation_results} shows this ablation, where bolded results represent the best performance within the horizontal lines of that section and $^{\dag\dag}$ corresponds to the best performance overall for the respective architecture. For `InfoNCE-LM' and `InfoNCE + LM', $\beta = 0.01$ corresponds to a slight mixing of negative pair latent representations and $\beta = 0.5$ leads to a U-shape probability density function in $[0, 1]$, leading to increased LM.

\begin{figure}
\centering
\includegraphics[scale=0.5]{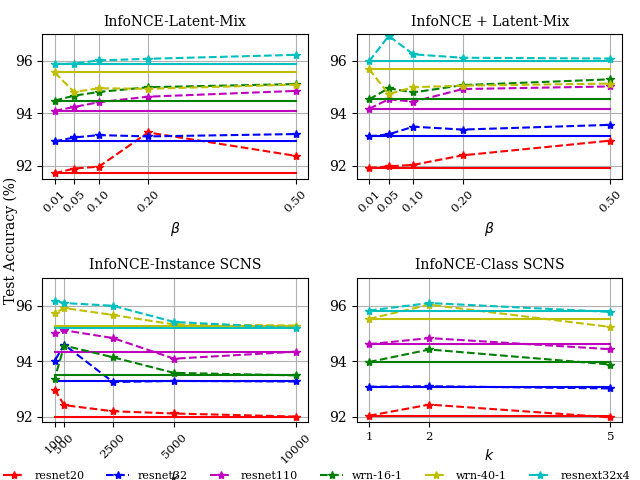}
\caption{Effect of $\beta$ in Latent Mixup and $k$ in SCNS}\label{fig:scns_analysis}
\end{figure}
We find that using the original negative samples and the latent mixture features (InfoNCE + LM) improves over only using the LM features (InfoNCE-LM). We also find that using LM with label smoothing improves cross-entropy training over only using cross-entropy training and that increased LM ($\beta = 0.5$) improves performance for both InfoNCE + LM and InfoNCE-LM. Hence, this suggests LM performs well for both point-wise and pairwise based supervised training. We note that the regularization term for the LM loss is manually searched over settings $\alpha \in [0.01, 0.05, 0.1, 0.2]$ on a validation dataset which is 5\% randomly sampled from the predefined CIFAR-10 training data. The results reported in \autoref{tbl:cifar10_ablation_results} are with a regularization term set $\alpha = 0.1$.~\autoref{fig:scns_analysis} visualizes the effect of changing $\beta$ for LM  and $k$ for Class-SCNS and Instance-SCNS for 5 different ResNet-based architectures. We find that on average InfoNCE + LM models outperforms InfoNCE-LM that setting $\beta \approx 0.5$. For Instance-SCNS, we find that $k = 500$ for Wide-ResNet architectures and $k = 100$ for ResNet architectures leads to increases in test accuracy.  

\paragraph{Knowledge Distillation} We now discuss results of the KD experiments on CIFAR-100~\cite{krizhevsky2009learning}.~\autoref{fig:cifar100_heatmap} shows the test accuracies for each KD method where the student-teacher pair is different for each row and the colours correspond to [0-1] row normalized test accuracies to visualize the relative percentage increase or decrease between each KD method\footnote{The naming conventions of baselines, including \emph{our own} KD baselines are described in \autoref{sec:architectures}.} 
\begin{table*}[ht]
\centering

\resizebox{2.\columnwidth}{!}{%

\begin{tabular}{ll|llllllllll}
\toprule
& \thead{Teacher Network \\ Student Network}  & 
\thead{resnet152\\resnet56} & 
\thead{resnet152\\resnet44} & 
\thead{resnet50 \\ resnet20} & 
\thead{resnet50 \\ resnet32} & 
\thead{resnet152 \\ resnet32} &
\thead{densenet121 \\ resnet32} &
\thead{resnet50 \\  resnet8} &
\thead{resnet50 \\  resnet14} &
\thead{resnet34 \\  resnet14} &
\thead{resnet50 \\  resnet32x4} \\

\toprule

& KLD~\cite{hinton2015distilling} & 52.22 & 51.29 & 47.80 &50.69 & 50.90 &   48.61 & 42.05 &44.91 &45.63 &  $62.44^{\dag\dag}$ \\

\cmidrule{2-12}
\cmidrule{2-12}

\parbox[t]{2mm}{\multirow{7}{*}{\rotatebox[origin=c]{90}{\textbf{SoTA Baselines}}}}

& AT~\cite{zagoruyko2016paying}  & 43.81 (-8.41) & 47.45 (-3.84) & 45.48 (-2.32) & 47.47 (-3.89) & 47.01 (-2.22) &  46.39 (-1.82) & 39.45 (-2.60) & 44.56 (-0.35) &45.20 (-0.43) &  54.49 (-7.95) \\

& SP~\cite{tung2019similarity} & 48.12 (-4.10) & 47.29 (-4.05) & 44.30 (-3.50) & 45.50 (-4.45) & 46.45 (-2.26) & 46.35 (-1.86) & 38.47 (-3.57) & 43.51 (-1.39) & 44.08 (-1.55) & 53.07 (-9.37) \\

& PKT~\cite{passalis2018learning} & 49.14 (-3.08) & 48.82 (-2.47) & 46.42 (-1.38) & 48.03 (-3.21) & 47.69 (-0.39) & 48.22 (+0.01) & 40.10 (-1.95) & 45.18 (\textbf{+0.27}) & 45.05 (-0.58) &  55.00 (-7.43) \\

& CC~\cite{peng2019correlation}  & 42.35 (-9.87) & 45.27 (-6.02) & 46.59 (-1.21) & 46.56 (-4.15) & 46.75 (-1.41) & 47.20 (-1.01) & 40.08 (-1.97) & 45.02 (\textbf{+0.11}) & 44.10 (-1.53) & 49.06 (-13.37) \\

& VID~\cite{ahn2019variational} & 48.49 (-3.73) & 48.12 (-3.17) & 45.59 (-2.21) & 47.63 (-2.81) & 48.09 (-4.98) & 43.63 (-4.57) & 40.65 (-1.40) & 44.40 (-0.51) & 45.45 (-0.18) &  54.75 (-7.69) \\

& FT~\cite{kim2018paraphrasing} & 44.19 (8.03) & 47.59 (-3.70) & 46.03 (-1.77) & 45.28 (-3.72) & 47.18 (-9.57) & 39.04 (-9.17) & 39.67 (-2.38) & 38.48 (-6.43) & 44.25 (-1.38) &  52.31 (-10.12) \\

& CRD~\cite{tian2019contrastive}  & 50.42 (-1.79) & 47.98 (-3.31) & 46.92 (-0.88) & 49.15 (-1.54) & 49.35 (-1.75) &  49.21 (\textbf{+0.60}) & 39.86 (-2.18) & 44.34 (-0.57) & 45.28 (-0.35) &  56.91 (-5.53) \\

\cmidrule{1-12}
\cmidrule{1-12}

\parbox[t]{2mm}{\multirow{9}{*}{\rotatebox[origin=c]{90}{\textbf{Ours}}}}

& Siamese-PC & 51.99 (-0.23) & 47.11 (-4.18) & 45.64 (-2.16) & 46.42 (3.99) & 46.91 (-1.81) & 46.80 (-1.41) & 38.03 (-4.02) & 43.83 (-1.08) &44.59 (-1.04) & 52.23 (-10.21) \\

& Contrastive-PC & 45.18 (-7.03) & 47.19 (-4.10) & 45.52 (-2.28) & 46.64 (-0.01) & 50.89 (-1.40) & 47.21 (-1.0) & 39.40 (-2.65) & 44.91 (0.0) & 44.09 (-1.54) & 54.22 (-8.21) \\

& Contrastive-CKA (RBF) & 47.57 (-4.65) & 48.71 (-2.58) & 46.29 (-1.51) & 46.87 (-3.59) & 47.31 (-1.54) & 47.07 (-1.14) & 41.28 (-0.77) & 47.05 (\textbf{+2.14}) & 44.76 (-0.87) &  52.50 (-9.93) \\

& Contrastive-CKA (Linear)  & 47.87 (-4.35) & 47.90 (-3.39) & 45.89 (-1.91) & 47.42 (-3.27) & 41.13 (-1.62) & 46.99 (-1.21) & 40.03 (-2.02) & 44.63 (-0.28) & 44.51 (-1.12) &  54.16 (-8.28)\\

\cmidrule{2-12}
\cmidrule{2-12}

& CRD-LM & 52.00 (- 0.22) & 50.28 (-1.01) & 48.12 (\textbf{+0.32}) & 50.23 (-1.22) & 49.68 (\textbf{+1.10}) & 49.71 (\textbf{+1.50}) & 42.18 (\textbf{+0.13}) & 46.08 (\textbf{+1.17}) &45.73 (\textbf{+0.09}) &  58.53 (-3.91) \\

& InfoNCE + Class-SCNS & 53.03 (+0.61) & 49.89 (-1.40) & 47.02 (-0.78) & 49.80 (-1.28) & 49.62 (\textbf{+ 0.60}) & 49.22 (\textbf{+0.61}) & 41.52 (-0.53) & 45.39 (\textbf{+0.48}) & 45.12 (-0.51) & 58.36 (-4.08) \\

& InfoNCE + Instance-SCNS & 52.24 (\textbf{+ 0.02}) & 51.13 (-0.16) & 47.44 (-0.63) & 50.04 (-.75) & 50.27 $(\textbf{+1.77})^{\dag\dag}$ & 50.38 (\textbf{+1.77}) & 41.29 (-0.76) & 46.12 (\textbf{+1.21}) & 45.80 (\textbf{+0.17}) & 58.95 (-3.49) \\

& InfoNCE-LM + Class-SCNS & 52.83 (\textbf{+ 0.61}) & 51.33 (\textbf{+0.04}) & 48.72 (\textbf{+0.92}) & 50.83 (-0.75) & 50.15 (\textbf{+1.29}) & 49.90 (\textbf{+1.69}) & 43.44 (+\textbf{1.39}) & 46.99 (\textbf{+2.08}) & 46.01 (\textbf{+0.38}) & 59.25 (-3.19) \\

& InfoNCE-LM + Instance-SCNS & 53.21 $(\textbf{+ 0.99})^{\dag\dag}$ & 52.72 $(\textbf{+1.43})^{\dag\dag}$ & 49.07 $(\textbf{+1.27})^{\dag\dag}$ & 51.18 (-0.48) & 50.42 (\textbf{+1.50}) & 50.11 $(\textbf{+1.89})^{\dag\dag}$ & 43.89 $(\textbf{+1.84})^{\dag\dag}$ & 47.32 $(\textbf{+2.41})^{\dag\dag}$ & 46.83 $(\textbf{+1.20})^{\dag\dag}$ & 60.09 (-2.35) \\

\bottomrule
\end{tabular}%
}

\caption{
\small{
Test \emph{accuracy} (\%) of student networks on Tiny-ImageNet-200 
}}
\label{tbl:tiny_imagenet_negative_sampling_results}
\end{table*}
We find that combining LM with InfoNCE + Instance-SCNS with a lookup table outperforms only using LM or instance-level SCNS with a lookup table. Moreover, `InfoNCE-LM + Class-SCNS' outperforms all other KD methods for all but one student-teacher pairing (PKT found to have the highest test accuracy for `resnet14-wrn-16-1' student-teacher pair). 
However, the original KLD distillation loss remains a very strong baseline that is competitive with CL and even outperforms our proposed non-contrastive baselines. We find a 0.24 point increase in the 0-1 normalized average score (`Kullbeck Leibler Distillation' = 0.75 and `InfoNCE-LM + Class-SCNS'=0.99).
We also find w.r.t. the student-teacher network capacity gap, increasing the capacity of the teacher network does not necessarily lead to improved student network performance if the gap is large. The performance difference between `resnet14-wrn-16-1' and `resnet14-wrn-40-1' is relatively small and `resnet14-wrn-40-1' has higher accuracy than `resnet14-wrn-16-1' in only 7/16 different loss functions. However, in 3/4 of the CL cases (4 rightmost columns of ~\autoref{fig:cifar100_heatmap}) the larger teacher network in `resnet14-wrn-40-1' has significantly improved accuracy. 
\autoref{fig:ctime} shows the convergence time comparing InfoNCE when using USNs and SCNS for the resnet32-wrn-16-1 student-teacher pair on CIFAR-100. We see that Instance-SCNS converges after 106 training epochs, Class-level SCNS at 124 epochs while USNs converges at 181 epochs. Hence, both test accuracy and convergence time is improved by sampling hard negative samples via SCNS. 

\begin{figure}[ht]
\centering
\includegraphics[scale=0.48]{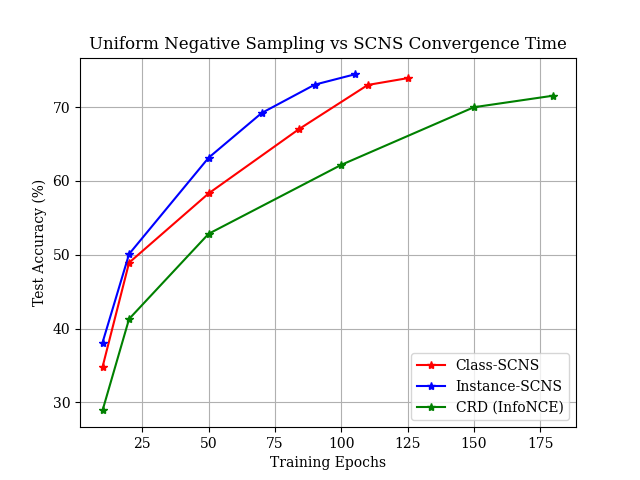}
\caption{CIFAR-100 Convergence Time Comparison}\label{fig:ctime}
\end{figure}

\paragraph{Distilling Transfer Learned Representations}
To test how SCNS performs in the transfer learning setting, we learn a student network from teacher network that takes inputs and outputs of different sizes. 
We use pretrained ImageNet models and fine-tune them on Tiny-ImageNet-200 by replacing the last 1000 dimensional linear layer with a 200 dimensional layer. The pretrained models are fine-tuned by resizing Tiny-ImageNet-200 images from 64x64 to 256x256 without any additional data augmentation. These models are used as the teacher networks that take in 256x256 images while the student network takes the original 64x64 input. This KD setup is slightly different as the teacher network is fine-tuned using transfer learning from the original ImageNet dataset, not from random initialization. 

\begin{table}[ht]
\centering
\resizebox{.5\textwidth}{!}{\begin{minipage}{\textwidth}

\begin{tabular}{l|llllll}
\toprule

\thead{Teacher Network \\ Student Network} &
\thead{resnet110\\resnet20} & 
\thead{resnet110\\resnet8} & 
\thead{wrn-16-1 \\ resnet14} & 
\thead{wrn-16-1 \\ resnet32} & 
\thead{wrn-40-1 \\ resnet20} &
\thead{wrn-40-1 \\  resnet14} \\

\midrule
\midrule
Teacher (Cross-Entropy) & 74.31 & 74.31 & 73.11 & 73.40 & 75.43 &  75.43  \\
Student (Cross-Entropy) & 66.26 & 57.31 & 65.84 & 69.80 & 67.69 & 68.02 \\
Student (Cross-Entropy + KLD) & 69.78 & 60.48 & 67.58 & 71.10 & 68.87 & 72.72 \\

\midrule
InfoNCE & & & & & & \\
\midrule
-LM ($\beta = 0.01$) & 68.13 & 58.32 & 67.11 & 70.09 & 69.02 & 68.55 \\
-LM ($\beta = 0.05$) & 68.15 & 58.41 & 67.25 & 69.98 & 69.27 & 68.80 \\
-LM ($\beta = 0.1$)  & 68.22 & 58.89 & 67.48 & 70.49 & 69.49 & 69.39 \\
-LM ($\beta = 0.2$) & 68.29 & \textbf{59.13} & 67.59 & \textbf{71.42} & 69.88 & 69.97 \\
-LM ($\beta = 0.5$) & \textbf{68.49} & 59.07 & \textbf{67.63} & 69.66 & \textbf{69.91} & \textbf{70.90} \\

\midrule
\midrule
+LM ($\beta = 0.01$) & 68.71 & 58.31 & 67.23 & 70.33 & 68.92 & 68.52 \\
+LM ($\beta = 0.05$) & 68.83 & 58.78 & 67.61 & 70.57 & 69.10 & 68.80 \\
+LM ($\beta = 0.1$)  & 68.96 & 59.01 & 68.45 & \textbf{70.92} & 69.48 & 69.04 \\
+LM ($\beta = 0.2$) & 69.08 & \textbf{61.09} & 69.58 & 70.83 & 71.03 & 70.24 \\
+LM ($\beta = 0.5$)  & \textbf{69.25} & 60.70 & \textbf{70.09} & 69.66 & \textbf{71.29} & \textbf{71.43} \\

\midrule
\midrule
-Class SCNS ($k = |V|/100$) & 70.03 & 63.20 & 68.48. & 69.31 & 69.44 & 70.41 \\
-Class SCNS ($k = |V|/50$)  & 70.88 & \textbf{64.03} & \textbf{69.73} & 71.09 & 69.62 & \textbf{72.83} \\
-Class SCNS ($k = |V|/20$) & $\textbf{71.32}^{\dag\dag}$ & 63.91 & 69.36 & 71.56 & $\textbf{71.37}^{\dag\dag}$ & 72.21 \\
-Class SCNS ($k = |V|/10$) & 71.03 & 62.14 & 68.73 & \textbf{72.02} & 71.02 & 72.08 \\
-Class SCNS ($k = |V|/5$) & 69.35 & 61.51 & 68.05 & 70.99 & 69.81 & 71.76 \\

\midrule
\midrule
-Instance SCNS ($k = |D|/100$)  & 70.24 & 63.69 & 68.92 & 70.01 & 69.75 & 70.20 \\
-Instance SCNS ($k = |D|/50$) & 71.21 & 64.99 & 69.23 & $\textbf{72.94}^{\dag\dag}$ & 69.70 &   $\textbf{73.10}^{\dag\dag}$ \\
-Instance SCNS ($k = |D|/20$) & \textbf{71.25} & $\textbf{65.13}^{\dag\dag}$ & $\textbf{70.11}^{\dag\dag}$ & 71.87 & 69.31 & 72.09 \\
-Instance SCNS ($k = |D|/10$) & 71.04 & 64.87 & 70.03 & 71.83 & \textbf{70.14} & 72.02 \\
-Instance SCNS ($k = |D|/5$) & 70.93 & 64.84 & 69.82 & 70.53 & 69.31 & 71.80 \\
\bottomrule

\end{tabular}
\end{minipage}}

\caption{
\small{
An ablation of efficient NS techniques on CIFAR-100. 
}}
\label{tbl:cifar100_ablation_results}
\end{table}
\autoref{tbl:tiny_imagenet_negative_sampling_results} shows each KD technique along with our proposed techniques from row `Contrastive-PC' to the last row. In almost all student-teacher network combinations, Instance-SCNS, Class-SCNS samples and both with Contrastive LM regularization have led to performance improvements over all previously proposed KD methods. However, the original KLD distillation loss~\cite{hinton2015distilling} remains a very strong baseline. We also find that increasing the capacity of the teacher network for the same sized student network can result in the same or poorer performance if the original student-teacher network capacity difference is large. For example, if we compare `-resnet14-wrn-16-1' to `resnet14-wrn-40-1" we can see there is little difference in performance across the different KD methods. However, increasing the student network size closer to the teacher network leads to improved performance e.g `resnet32-wrn-16-1' consistently improves over `resnet14-wrn-16-1'.

~\autoref{tbl:cifar100_ablation_results} shows the ablation of all three proposed methods on CIFAR-100. The most consistent gain in performance is found when using Instance-SCNS as it achieves the best performance for 4 out of 6 student networks. Class-SCNS performs the best for resnet20 student networks, which have relatively larger capacity compared to resnet8 and resnet14. 


\section{Conclusion}
We proposed (1) semantically conditional negative sampling, a method that use pretrained networks to define a negative sampling distribution and (2) latent mixup, a simple strategy to form \textit{hard} negative samples. We found that when used in a contrastive learning setting, both proposals consistently outperform previous knowledge distillation methods and improve contrastive learned models in the standard supervised learning setup.

\bibliography{icml2021}
\bibliographystyle{icml2021}

\appendix

\section{Hardware Details}\label{sec:hardware}
All experiments were run on a Titan RTX P8 24G memory graphics processing unit. 

\section{Network Architectures}\label{sec:architectures}
The below CNN architectures are used for standard supervised learning experiments on CIFAR10 and for KD experiments on CIFAR100 and Tiny-ImageNet-200. 

\begin{itemize}[topsep=0pt,itemsep=-0.5ex,partopsep=1ex,parsep=1ex]
    \item Wide Residual Network~\citep[WRN;][]{zagoruyko2016wide}. WRN-d-w represents wide resnet with depth d and width factor w.
    \item ResNet~\cite{he2016deep}. We use resnet-d to represent cifar-style resnet with 3 groups of basic blocks,
each with 16, 32, and 64 channels respectively. In our experiments, resnet8x4 and resnet32x4 indicate a 4 times wider network (namely, with 64, 128, and 256 channels for each of the block).
    \item ResNet~\cite{he2016deep}. ResNet-d represents ImageNet-style ResNet with Bottleneck blocks and more channels.
    \item ResNeXt~\cite{xie2017aggregated}. ResNeXt  ImageNet-style ResNet with Bottleneck blocks and more channels.
    \item VGG~\cite{simonyan2014very}. The pretrained VGG network are adapted from its original ImageNet counterpart.
    \item DenseNet~\cite{huang2017densely}. We use a pretrained ImageNet DenseNet-121 and fine-tune on Tiny-ImageNet-200 with upscale images (64x64 to 256x256). 
\end{itemize}

\begin{table*}[ht]
    \centering
\resizebox{1.85\columnwidth}{!}{%
\begin{tabular}{l|l}
Abbreviation & KD Method \\
\hline
KLD & Knowledge Distillation~\citep[KD;][]{hinton2015distilling} - \textbf{Kullbeck Leibler Divergence} \\ 
Fitnets & \textbf{Fitnets}: Hints for thin deep nets~\cite{romero2014fitnets} \\
AT & \textbf{Attention Transfer}~\citep[AT;][]{zagoruyko2016paying}\\
SP & \textbf{Similarity-Preserving} Knowledge Distillation~\citep[SP;][]{tung2019similarity}\\
CC & \textbf{Correlation Congruence}~\citep[CC;][]{peng2019correlation}\\
VID & \textbf{Variational information distillation} for knowledge transfer~\citep[VID;][]{ahn2019variational} \\
RKD & \textbf{Relational Knowledge Distillation}~\citep[RKD;][]{park2019relational} \\
PKT & Learning deep representations with \textbf{probabilistic knowledge transfer}~\citep[PKT;][]{passalis2018learning}  \\
FT & Paraphrasing complex network: Network compression via \textbf{factor transfer}~\citep[FT;][]{kim2018paraphrasing}  \\ 
FSP & A gift from knowledge distillation: Fast optimization, network minimization and transfer learning~\citep[FSP;][]{yim2017gift}  \\
\bottomrule
\end{tabular}}
\caption{Abbreviations for Knowledge Distillation Baselines}
\label{tab:kd_types}
\end{table*}

\section{Experiment Details}\label{sec:experiment_details}

\subsection{Conditional Negative Sampling Details}
Before running experiments for supervised learning and knowledge distillation, we must first define the negative sampling distribution on the instance-level or class-level. For class-level SCNS we use cross-modal transfer by using word embeddings for the class labels. We use $\texttt{skipgram}$ word vectors~\cite{mikolov2013distributed} that are pretrained on GoogleNews and can be retrieved from \url{https://code.google.com/archive/p/word2vec/}. For class labels that are phrases, we average the pretrained word embeddings of each constituent embedding prior to computing cosine similarity. 
We find best results for our proposed method with the temperature $\tau=5$ when constructing $\mat{P}$. This ensures that the distribution is not too flat and encourages tighter coupling of neighbours.

For instance-level SCNS, pair similarity is defined by a pretrained network of the same type that is used for training in the supervised learning setting. For knowledge distillation, the teacher network is used to define the pair similarity.

\subsection{Dataset and Model Details}
For all models used in the standard supervised learning and KD settings, we use the cross-entropy loss optimized using Stochastic Gradient Descent (SGD) with a decay rate (different setting for each task). Additionally, hyperparameter tuning of $\beta$ and $k$ is tested on a randomly sampled 5\% of the predefined training data of all three datasets. 

\textbf{CIFAR-10} For CIFAR-10, the learning rate is set to 0.01, momentum=0.9 and weight\_decay=0.0005. The images are randomly cropped and horizontally flipped and normalized along the input channels (as two tuple arguments $(0.4914, 0.4822, 0.4465), (0.2023, 0.1994, 0.2010)$ in the \texttt{transforms.Normalize} method in the \texttt{torchvision} library). The batch size is 128 for training. 

\textbf{CIFAR-100} For CIFAR-100 a dataset with 50k training images (500 samples per 100 classes) and 10k test images, the learning rate is set to 0.05 with a decay rate 0.1 at every 25 epochs after 100 epochs until the last 200 epoch. The batch size is set to 64. 

\textbf{Tiny-ImageNet-200} For Tiny-ImageNet-200, we train for 100 epochs and decay the learning rate every 20 epochs. The batch size is also set to 64. The student is trained by a combination of cross-entropy classification objective and a KD objective as $\ell = \ell_{\mathrm{CE}} + \alpha\ell_{KD}$.

\paragraph{Baseline KD Settings}
The abbreviations refer to correspond to the KD method names listed in \autoref{tab:kd_types}. 

The main KD influence factor $\alpha$ is set based on either the original paper settings or a set using a grid search over few settings close to the original paper settings. For our proposed Pearson Correlation (PC) and Centered Kernel Alignment (CKA) KD objectives, we grid search over $\alpha \in [0.2, 0.5, 0.65, 0.8, 0.9]$ and $\zeta \in [0.1, 0.2, 0.5,  0.7, 0.9]$ for those objectives that use a margin-based triplet loss (e.g Triplet CKA). The parameter settings specified in the paper of previous KD methods is used and where not specified we manually grid search different settings of $\alpha$.

\begin{enumerate}[topsep=0pt,itemsep=-0.5ex,partopsep=1ex,parsep=1ex]
\item Fitnets~\cite{romero2014fitnets}: $\alpha = 80$
\item AT~\cite{zagoruyko2016paying}: $\alpha = 600$
\item SP~\cite{tung2019similarity}: $\alpha = 2000$
\item CC~\cite{peng2019correlation}: $\alpha = 0.05$
\item VID~\cite{ahn2019variational}: $\alpha = 0.8$
\item RKD~\cite{park2019relational}: $\alpha_1 = 25$ for distance and $\alpha_2 = 50$ for angle. For this loss, we combine both term following the original paper.
\item PKT Probabilistic Knowledge Transfer~\cite{passalis2018learning}
~\cite{passalis2018learning}: $\alpha = 10000$
\item AB~\cite{heo2019knowledge}: $\alpha = 0.2$, distillation happens in a separate pre-training stage where only distillation objective applies.
\item FT~\cite{kim2018paraphrasing}: $\alpha = 500$
\item FSP~\cite{yim2017gift}: $\alpha = 0$, distillation happens in a separate pre-training stage where
only distillation objective applies.
\item NST~\cite{huang2017like}: $\alpha = 50$
\item CRD Contrastive Representation Distillation~\cite{tian2019contrastive}: $\alpha = 0.8$, in general $\alpha \in [0.5, 1.5]$ works well.
\item Kullbeck-Leibler Divergence Distillation (KLD)~\cite{hinton2015distilling}: $\alpha = 0.9$ and $\tau = 4$.
\end{enumerate}

\paragraph{Our Proposed KD Baseline Method Settings}

\begin{enumerate}[topsep=0pt,itemsep=-0.5ex,partopsep=1ex,parsep=1ex]
\item Siamese-PC: This is the loss from \autoref{eq:pc_loss}. $\alpha = 0.8$ 
\item Triplet CKA (Linear) or referred to as Linear-CKA: This is the loss from \autoref{eq:triplet_cka} loss with a linear kernel - $\alpha = 0.8$ and $\zeta = 0.2$
\item Triplet CKA (RBF) or referred to as Linear-RBF: $\alpha = 0.8$ and $\zeta = 0.2$
\item Contrastive-CKA described in \autoref{eq:triplet_cka}: $\alpha = 0.15$
\item Contrastive-PC: This is the loss in \autoref{eq:pc_loss} applied to both the positive pair embeddings and negative sample pair embeddings. 
\end{enumerate}

\paragraph{Our Proposed SCNS-Based KD Settings}

\begin{enumerate}
\item InfoNCE Loss with Instance-level SCNS: $\alpha = 0.9$
\item InfoNCE-LM: This is the InfoNCE loss between latent mixup representations as in Equation 9. 
\item InfoNCE+Class SCNS:  This represents SCNS with an InfoNCE loss. $\alpha = 0.65$
\item InfocNCE-LM + Class-SCNS: This represents SCNS with an InfoNCE loss with an second InfoNCE loss for latent mixup representations. $\alpha = 0.65$

\end{enumerate}

\paragraph{Preprocessing Details}
For experiments with the CKA objective we group mini-batches by their targets as CKA operates on cross-correlations between samples of the same class. Therefore, random shuffling is carried out on the mini-batch level but not on the instance level. For all other objectives, standard random shuffling of the training data is performed. 
\subsection{Metric Learning Distillation Objectives}
In our work we also propose two correlation and kernel-based loss functions that can be used for both standard pointwise-based KD and metric-learned KD. These are used as alternatives from those described in the related research, which we describe below. 

\paragraph{Metric-based Centered Kernel Alignment}
The CKA function measures the closeness of two set of points that can be represented as matrices. Thus far it has only been used for analysing representation similarity in neural networks~\cite{kornblith2019similarity} but not for optimizing a neural network. We propose to distil the knowledge of the teacher network by minimizing the alignment between student and teacher representations using CKA as a baseline.

For two arbitrary matrices $\mat{Z}_i \in \mathbb{R}^{M \times d_L}$ and $\vec{Z}_j \in \mathbb{R}^{M \times d_L}$, each consisting of a set of neural network representations, the centered alignment (CA) can be expressed as, 

\begin{equation}
    \mathrm{CA}(\mat{Z}_{i}, \mat{Z}_{j}) = \frac{\langle \mathrm{vec}(\mat{Z}_{i}\mat{Z}_i^{\top}), \mathrm{vec}(\mat{Z}_{j}\mat{Z}_j^{\top}) \rangle}{ ||\mat{Z}_{i}\mat{Z}_i^{\top}||_F ||\mat{Z}_{j}\mat{Z}_j^{\top}||_F}
\end{equation}

where $||\cdot||_F$ is the Frobenius norm. We can replace the  the dot product in numerator with a kernel function $K(\cdot,\cdot)$ to compute the CKA. The kernel function is smooth and differentiable, hence we use it as a loss function for KD-based metric learning to maximize the similarity between the positive class latent representations given by $(z_{+}^{\mathcal{S}}, z_{+}^{\mathcal{T}})$ and negative class latent representations $(z_{-}^{\mathcal{S}}, z_{-}^{\mathcal{T}})$.~\autoref{eq:cka_def} shows the formulation of CKA where a kernel is used instead of the dot product.

\begin{equation}\label{eq:cka_def}
\begin{split}
    \mathrm{CKA}(\mat{Z}_i, \mat{Z}_j) = & \mathrm{K}(\mat{Z}_i, \mat{Z}_j) - \mathbb{E}_{\mat{Z}}[\mathrm{K}(\mat{Z}_i, \mat{Z}_j)] - \\ 
      \mathbb{E}_{\mat{Z}_j}[\mathrm{K}(\mat{Z}_i, \mat{Z}_j)] & + \mathbb{E}_{\mat{Z}_i, \mat{Z}_j}[\mathrm{K}(\mat{Z}_i, \mat{Z}_j)]
\end{split}
\end{equation}

In our experiments, a linear kernel and a radial basis function ($K(\mat{Z}_i, \mat{Z}_j) = \exp(- ||\mathrm{vec}(\mat{Z}_i) - \mathrm{vec}(\mat{Z}_j)||^{2}_F/2S^{2})$) were used where $S^2$ is the sample variance. To account for intra-variations and inter-variations between between student and teacher representations the
CKA loss is used as apart of a triplet loss that maximizes the kernel similarity between the positive pair of the student anchor and student positive sample and also the student anchor with the teacher anchor. This is shown as $\ell_{\mathrm{CKA}}^{+}$ in \autoref{eq:triplet_cka}, where $z_{*}$ represents the anchor sample. The same is computed for the negative pair, denoted by $\ell_{\mathrm{CKA}}^{-}$. Both losses are combined as one where $\zeta$ controls the tradeoff between positive pair losses and negative pair losses and $m$ is the margin.  


\begin{equation}\label{eq:triplet_cka}
\begin{split}
    \ell^{+}_{\mathrm{CKA}} = & \mathrm{CKA}(\vec{z}^S_{+}, \vec{z}^S_{*}) + \mathrm{CKA}(\vec{z}^S_{*}, \vec{z}^T_{*}), \\
     \ell^{-}_{\mathrm{CKA}} = & \mathrm{CKA}(\vec{z}^S_{-}, \vec{z}^S_{*}) + \mathrm{CKA}(\vec{z}^S_{-}, \vec{z}^T_{*}), \\
     \ell_{\mathrm{CKA}} = & \max \big(0, \zeta \ell^{+}_{\mathrm{CKA}} -  
    (1 -\zeta) \ell^{-}_{\mathrm{CKA}} + m\big) 
\end{split}
\end{equation}

Concretely, this will force all positive class representations to have high a CKA score within and across the samples for the student and teacher representations, and similarly for the negative pair of the triplet.

\paragraph{Pearson Correlation Representation Distillation}
An alternative to maximizing the mutual information between $z^{\mathcal{S}}$ and $z^{\mathcal{T}}$~\cite{belghazi2018mutual} is instead to maximize the linear interactions using a PC-based loss as a strong baseline. The objective to be maximized is expressed as \autoref{eq:pc_loss}

\begin{equation}\label{eq:pc_loss}
\begin{split}
\ell_{\mathrm{PC}}^{+} = &
\rho_{\mathrm{PC}}(\vec{z}^{\mathcal{S}}_{-}, \vec{z}^{\mathcal{T}}_{-}) + \rho_{\mathrm{PC}}(\vec{z}^{\mathcal{S}}_{+}, \vec{z}^{\mathcal{T}}_{+}),  \\
\ell_{\mathrm{PC}}^{-} = & \rho_{\mathrm{PC}}(\vec{z}^S_{-}, \vec{z}^S_{+}) + \rho_{\mathrm{PC}}(\vec{z}^S_{-}, \vec{z}^T_{+}), \\
\ell_{\mathrm{PC}} = & \max \big(0, \zeta \ell^{+}_{\mathrm{PC}} -  
(1 -\zeta) \ell^{-}_{\mathrm{PC}} + m\big) 
\end{split}
\end{equation}

where $\rho_{\mathrm{PC}} \in [-1, 1]$ computes the correlation coefficient. When using the $\ell_{\mathrm{PC}}$ loss with contrastive learning (`Contrastive-PC') we take the average loss as $\frac{1}{N-1}\sum_{i=1}^{N-1}\rho_{\mathrm{PC}}(z^{\mathcal{S}}_{-, i}, z_{+}^{\mathcal{S}})$ and similarly for the remaining losses that use negative sample representations.

\section{Connection Between Mutual Information \& Conditional Negative Sampling}\label{sec:mi_and_scns}
In this section we describe contrastive learning with our proposed conditional negative sampling in terms of mutual information (MI). 
Let $p(y)$ be the probability of observing the class label $y$ and $p(x, y)$ denote the probability density function of the corresponding joint distribution. Then, the MI is defined as \autoref{eq:mi} 

\begin{equation}\label{eq:mi}
I(X ; Y) = \sum_y \int_x p(x, y) \log \frac{p(x,y)}{ p(x)p(y)} dx
\end{equation}

and can be further expressed in terms of the entropy $\mathbb{H}(X)$ and conditional entropy $\mathbb{H}(X|Y)$ as shown in \autoref{eq:mi_entropy}.
\begin{equation}\label{eq:mi_entropy}
\begin{gathered}
    I(X; Y) =  \sum_y \int_x p(x, y) \log \frac{p(x|y)}{ p(x)} dx  \\ 
= - \int_x p(x) \log p(x) - ( - \int_x \sum_y \log p(x, y) p(x| y)) \\
= \mathbb{H}(X) - \mathbb{H}(X|Y)
\end{gathered}    
\end{equation}

Then $I(X;Y)$ can be formulated as the KL divergence between $p(x, y)$ and the product of marginals $p(x)$ and $p(y)$,

\begin{equation}\label{eq:kl_mi}
\begin{gathered}
I(X; Y) \text{=} D_{\tiny{KL}} \big(p(x, y) || p(x) p(y)\big) \text{=} 
\underset{p(x, y)}{\mathbb{E}}\Big[\frac{p(x, y)}{p(x) p(y)}\Big]
\end{gathered}
\end{equation}

 Hence, if the classifier can accurately distinguish between samples drawn from the joint $p(x, y)$ and those drawn from the product of marginals $p(x)p(y)$, then $X$ and $Y$ have a high MI. However, estimating MI between high-dimensional continuous variables is difficult and therefore easier to approximate by maximizing a lower bound on MI. This is known as the InfoMax principle~\cite{linsker1988application}. In the below subsections, we describe how this MI lower bound is maximized using the InfoNCE loss~\cite{oord2018representation}.

\subsection{Estimating Mutual Information with InfoNCE}
The InfoNCE loss maximizes the MI between $\vec{z}_i$ and $\vec{z}_j$ (which is bounded by the MI between $\vec{z}_i$ and $\vec{z}_j$). The optimal value for $f(\vec{z}_j, \vec{z}_i)$ is given by $p(\vec{z}_j|\vec{z}_i)/p(\vec{z}_j)$. Inserting this back into Equation 4 and splitting $X$ into the positive sample and the negative examples $X_{-}$ results in:

\begin{equation}\label{eq:mi_lower_bound_ext}
\begin{gathered}
    \ell = - \underset{X}{\mathbb{E}} \log \Bigg[ \frac{\frac{p(x_{*}|x_i)}{p(x_{*})}}{\frac{p(x_{*}|x_{+})}{p(x_{*})} + \sum_{{x_i} \in X_{-}}\frac{p(x_i|x_{+})}{p(x_i)}} \Bigg] \\ 
    =\underset{X}{\mathbb{E}}\log \Bigg[1 + \frac{p(x_{*})}{p(x_{*}|x_{+})} + \sum_{x_i \in X_{-}} \frac{p(x_i|x_{+})}{p(x_i)}\Bigg]\\
    \approx\underset{X}{\mathbb{E}}\log \Bigg[1 + \frac{p(x_{*})}{p(x_{*}|x_{+})} (M-1) \underset{x_i}{\mathbb{E}}\frac{p(x_i|x_{+})}{p(x_i)}\Bigg] \\
    = \underset{X}{\mathbb{E}}\log \Big[1 + \frac{p(x_{*})}{p(x_{*}|x_{+})} (M-1) \Big] \\
    \geq \underset{X}{\mathbb{E}}\log  \Big[\frac{p(x_{*})}{p(x_{*}|x_{+})} M\Big] 
    = - I(x_{*}, x_{+}) + \log(M) 
\end{gathered}    
\end{equation}

Therefore, $I(z_j, z_i) \geq \log(N) - \ell$ which holds for any $f$, where higher $\ell$ leads to a looser MI bound. This MI bound becomes tighter as the number of negative sample pairs $M$ increases and in turn is likely to reduce $\ell$. In our work we argue that the $\log(M)$ term be replaced with a term that accounts for the geometry of the embedding space as not all negative sample pairs are equally important for reducing $\ell$. Therefore, the next subsection describes \emph{our} formulation of the MI bound that incorporates the notion of semantic similarity between embeddings of the sample pairs in order to choose an informative $M$ samples to tighten the bound.

\begin{figure}
    \centering
    \includegraphics[scale=0.4]{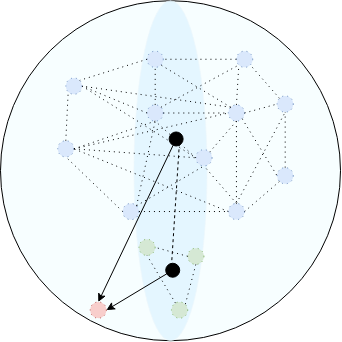}
    \caption{\small{$k$-NN Conditional Negative Sampling for embeddings on the $L_2$ sphere retrieved from the lookup table~\cite{xiao2017joint}. Blue circles are \textit{easy} negative sample embeddings, green circles are the top-$k$ most semantically similar embeddings for the target sample in green and black circles represent embedding centroids computed using a weighted average of embedding where the weights are the alignment score. }}\label{fig:scns_sphere}
\end{figure}

\subsection{Mutual Information Lower Bounds for Semantically Conditioned Negative Sampling}

We formalize the connection between SCNS and maximizing MI between representations in the contrastive learning and formulate an expression that describes how SCNS tightens the lower bound estimate on MI. We begin by defining `closeness' between representations $\{z_{*}, z_{+}, z_{-}\} \in \mathcal{Z}$ of samples $\{x_{*}, x_{+}, x_{-}\} \in \mathcal{X}$. Given that we are not restricted to a distance, divergence or similarity between vectors, we refer to a general measure as an \textit{alignment} function $A$ that outputs an alignment score $a \in [0, 1]$.

Given an anchor sample $\vec{x}_{*}$, the expected alignment score for the top-$k$ negative samples is $a^{k}_{x_{-}} := \mathbb{E}_{x_{-} \sim D^{k}_x} [A(z_{*}, z_{-})]$ and for the negative samples outside of the top-$k$ samples it is $a^{r}_{x_{-}} := \mathbb{E}_{x_{-} \sim D^{r}_x } [A(z_{*}, z_{-})]$ where $ D^{r}_{x_{*}} \subseteq D, D^{k}_{x_{*}}  \not\in D^r_{x_{*}}$, $r = N - N_y - k$. and $N_y$ is the number of samples of class $y$.

From the above, $\alpha \approx 1$ corresponds to negative samples that lie very close to $x_{*}$. We can then use the alignment weight (AW) $\Omega_{x_{*}}:= 1 - a^{k}_{x_{*}}/(a^{k}_{x_{*}} + a^{r}_{x_{*}})$ to represent the difference in `closeness' between the top-$k$ negative samples and the remaining negative samples. This is visualized as the difference in alignment between the centroids in the embedding space shown in \autoref{fig:scns_sphere}. 
 
 We can then replace $M$ negative samples as $\Omega := \sum_{i=1}^M \Omega_{x_i}$ and substitute $\log (M)$ with $\log (2\Omega)$ in \autoref{eq:mi_lower_bound_ext}. When the centroids of both negative samples sets are close (i.e $a^{k}_{x_{*}} = a^{r}_{x_{*}}$) then $\Omega = M/2$. Hence $\log(2\Omega) \approx \log(M)$ when the negative samples are centered in the same region, in which case uniform sampling provides the same guarantees as SCNS. The new MI lower bound is then \autoref{eq:new_mi_lb_ext}.

 \begin{equation}\label{eq:new_mi_lb_ext}
     I(X, Y) \geq \ell + \log(2\Omega)
 \end{equation}
 
 Intuitively, this bound favors top-$k$ negative samples that are close to the positive class boundary but are also relatively close compared to the remaining negative samples. This is dependent on how the embedding space is constructed from the pretrained network and which $A$ is chosen. Hence, this is clearly an estimation on the true MI lower bound. 
 
 We can now define the relation between $M$ uniformly sampled negatives (USNs) and top-$k$ SCNSs. Let $\Omega_{D_{x}^{\mathcal{U}}}$ be the AW of USNs for $x_{*}$ and $\Omega_{D^{k}_{x_{*}}}$ is the AW for the top-$k$ negative samples. When $|D_{x_{*}}^{\mathcal{U}}| \approx |D^{k}_{x_{*}}|$ (i.e both negative sample sets lie on a ring on the hypersphere around $x_{*}$),

\begin{equation}
    I(X, Y) \leq \ell + \log(2\Omega_{D^k_x}) \leq \ell + \log(2\Omega_{\mathcal{U}})
\end{equation}

given the non-uniform prior over negative samples as defined in SCNS. For some $k \ll N$, $2(\Omega_{k} - \Omega_{\mathcal{U}})= 0$ is met when a subset of $D^{\mathcal{U}}_{x}$ negative samples have $\bar{z}^{\mathcal{U}}_{\tilde{x}} \approx \bar{z}^{k}_{x}$ where the subscript $\tilde{x}$ denotes the aforementioned subset and $\bar{z}^{k}_{x}$ is the centroid of the top-$k$ negative samples

\section{Uniform vs SCNS Sample Complexity}\label{sec:scns_complexity}

In this section, we aim to identify the relationship between uniform sampling and SCNS by formulating how many draws are required to cover the the top-$k$ samples. We begin by the simpler case of single draws with a uniform distribution and extend it to batches of negative samples of size $b$ are drawn uniformly. We then repeat this with draws of unequal probability in \autoref{sec:ccp_unequal_prob}, as is the case for SCNS.  



\subsection{Number of Samples Until Observing Top-$k$ SCNS Samples Under a Uniform Distribution}\label{sec:ccp_single_ext}
We now describe the expected number of i.i.d drawn negative samples from $M$ to observe the top-$k$ samples at least once for a given $x$. Let $N_i$ denote the number of negative samples observed until you see the $i$-th new negative sample among the top-k samples and $N$ is the number of samples until all top-$k$ negative samples are observed. Since $N=\sum^{k}_{i=1} N_i$, 
\begin{equation}
\mathbb{E}[N] = \mathbb{E}\big[\sum_{i=1}^{k} N_i\big] = \sum_{i=1}^{k}\mathbb{E}\big[N_i\big]    
\end{equation}
where $N_i$ follows a geometric distribution with parameter $(k+1-i)/M$. Therefore $\mathbb{E}[N_i] = \frac{M}{k + 1 - i}$ and $\mathbb{E}[N]= M \sum_{i=1}^k (k+1-i)^{-1} = M\sum_{i=1}^k i^{-1}$.
\subsubsection{Number Of Batches To Cover All Samples}\label{sec:ccp_batch}

In \autoref{sec:ccp_single_ext}, we formulate the number of uniform negative samples required to cover the top-k negative samples at least once for a single consecutive draws. However, in practice mini-batch training is carried out and therefore it is necessary to reformulate
this for consecutive mini-batch draws of size $b$ for a given $x$ with replacement. This is a special case of the Coupon Collector's Problem. 

\begin{equation}
\begin{gathered}
    \sum_{j=0}^{\infty} P(K > ib) = \sum_{i=0}^{\infty} \big(1 - \frac{M!}{M^{ib}}\bracenom{ib}{M}  \big) \\
  = \sum_{i=0}^{\infty} \big(1 - \frac{1}{M^{ib}}\sum_{i=0}^{M} (-1)^{M-j}\binom{M}{j} l^{ib} \Big) \\
  = \sum_{j=0}^{\infty}\sum_{j=0}^{M-1}(-1)^{M-j+1}\binom{M}{j}\binom{j}{M}^{ib} \\
  = \sum_{j=0}^{M-1}(-1)^{M-j+1}\binom{M}{j}\frac{1}{1 - (\frac{j}{M})^{b}}
\end{gathered}
\end{equation}

Hence, as $M$ grows more mini-batch updates are needed until the top-$k$ \textit{hard} negative samples are observed. Thus, $b$ is required to be larger which is typically required in the MI formulation of NCE. To define the difference between USNs and SCNS we also need to define the CCP for unequal probabilities as defined by $\mat{P}_{x_{*}}$ for $x_{*}$. We first make a distributional assumption. Here, we assume that the distances (or \textit{alignment}) for $x$ to its $N$ negative samples follows a power law distribution. This is well-established for text~\cite{bollegala2010relational,piantadosi2014zipf} and we also observe a power law trend when computing the cosine similarities between all pairs with $f^{\mathcal{T}}$. 

\subsection{Number of Samples Until Observing Top-$k$ SCNS Samples Under an SCNS Distribution}\label{sec:ccp_unequal_prob}

In this subsection, we formulate the expected number of negative samples required for non-uniform sampling distributions, namely our proposed SCNS distribution provided by the teacher network. 

\paragraph{Maximum-Minimum Identity Approach}
The number of draws required to observe all top-$k$ NS is $C= \max\{C_1,\ldots,C_N\}$ where $N_i$ has a conditional probability $p_i$ of being sampled as defined in SCNS. Since the minimum of $N_i$ and $N_j$ is the number of negative samples needed to obtain either the $i$-th top-$k$ sample or the $j$-th top-$k$ sample, it follows that for $i \neq j$, $\min(N_i,N_j)$ has probability $p_i + p_j$ and the same is true for the minimum of any finite number of these random variables. The Maximum-Minimums Identity~\cite{ross2014first} is then used to compute the expected number of draws:

\begin{equation}
\begin{gathered}
    \mathbb{E}[N] = \mathbb{E}[\underset{i=1,\ldots M}{\max} N_i] =
    \sum_{i} \mathbb{E}[N_i] -
    \sum_{i < j} \mathbb{E}[\min(N_i, N_j)] \\
    +  \sum_{i < j < k} \mathbb{E}[\min(N_i, N_j, N_k)] - \ldots \\ 
    + (-1)^{M + 1} \mathbb{E}[\min(N_1, N_2, \ldots, N_M)] \\
\end{gathered}    
\end{equation}
We can then express the above in terms of the individual probabilities associated with drawing $M$ negative samples conditioned on a given $x_{*}$ as,
\begin{equation}
\begin{gathered}
\mathbb{E}[N] = \sum_i \frac{1}{p_i} - \sum_{i < j} \frac{1}{p_i + p_j} + \\
\sum_{i < j< k}\frac{1}{p_i + p_j + p_k} + (-1)^{M+1}\frac{1}{p_1 + \ldots + p_M}
\end{gathered}
\end{equation}

Since $\int^{\infty}_{0} e^{-px}dx = \frac{e^{-px}}{p}\Big|^{x=+\infty}_{x=0} = \frac{1}{p}$, integrating gives

\begin{equation}
\begin{gathered}
    1 - \prod_{i=1}^{N}(1 - e^{-p_i x}) =
    \sum_i e^{-p_i x} \\ = \sum_{i < j} e^{-(p_i + p_j)x} + \ldots +
    (-1)^{N + 1}e^{-(p_1 + \ldots + p_N)x}
\end{gathered}
\end{equation}

Hence, we get a concise equivalent expression~\cite{flajolet1992birthday}:
\begin{equation}
    \mathbb{E}[X_{-}] = \int^{+\infty}_{0}\Big(1 - \prod_{i=1}^{N}(1 - e^{-p_i x})\Big)dx
\end{equation}
The probability of sampling the $i$-th top-$k$ negative sample is $p_i \geq 0$ such that $p_1 + \ldots + p_N = 1$. To determine $\mathbb{E}[N]$, we first assume that the number of negative samples to draw $t$ as $X_{-}(t)$, follows a Poisson distribution with parameter $\lambda= 1$. Let $\mathbb{I}_i$ be the inter-arrival time between the $(i-1)$-th and the $i$-th negative sample draw: $\mathbb{I}_i$ has exponential distribution with parameter $\lambda = 1$. Let $Z_i$ be the time in which the $i$-negative sample arrives for the \emph{first} time (hence $Z_i \sim \exp(p_i)$) and let $Z = \max\{Z_1,\ldots,Z_N\}$ be the time in which we have observed all samples at least once.
Note that $Z=\sum^N_{i=0}\mathbb{I}_i$ and $\mathbb{E}[X] =\mathbb{E}[Z]$, indeed:

\begin{equation}
\begin{gathered}
\mathbb{E}[Z] = \mathbb{E}[\mathbb{E}[Z|N]] = \sum_k \mathbb{E}\Big[\sum_{i=1}^{k}\mathbb{I}_i|N = k\Big] P(N=k) \\
= \sum_k \Big[\sum_{i=1}^{k}\mathbb{I}_i\Big] \mathbb{P}(X=k) = \sum_{k}\sum_{i=1}^{k} \mathbb{E}[\mathbb{I}_i]\mathbb{P}(N = k) \\
\sum_k k \mathbb{P}(N=k) = \mathbb{E}(N)
\end{gathered}
\end{equation}

It follows that it suffices to calculate $\mathbb{E}[Z]$ to get $\mathbb{E}[N]$. Since $Z= \max\{Z_1,..,Z_N\}$, we have $
    F_{Z}(t) = \mathbb{P}(Z \leq t) = \prod_{i=1}^{N} F_{Z_i}(t) = \prod_{i=1}^{N}(1 - e^{-p_i t})
$ and then 
\begin{equation}
    \mathbb{E}[Z] = \sum_{0}^{+\infty} \mathbb{P}(Z > t)dt = \sum_{0}^{+ \infty} \Big(1 - \prod_{i=1}^{N}(1 - e^{-p_i t}) \Big) dt
\end{equation}

From the above expression, we clearly see that when $p_i$ is defined by a non-uniform distribution, the number of draws is proportionally larger in $N$. However, our original goal is to only sample from the most probably top-$k$ samples, in which case $\mathbb{E}[Z]$ is lower.

\section{Additional Results}

\autoref{fig:tiny_imagenet_boxplot} is a boxplot of how the performance changes for different KD methods as the student-teacher capacity gap varies. The purpose of this is to identify how much the performance increases are due to larger capacity as opposed to the particular KD method used. 

\begin{figure}[ht]
    \centering
    \includegraphics[scale=0.36]{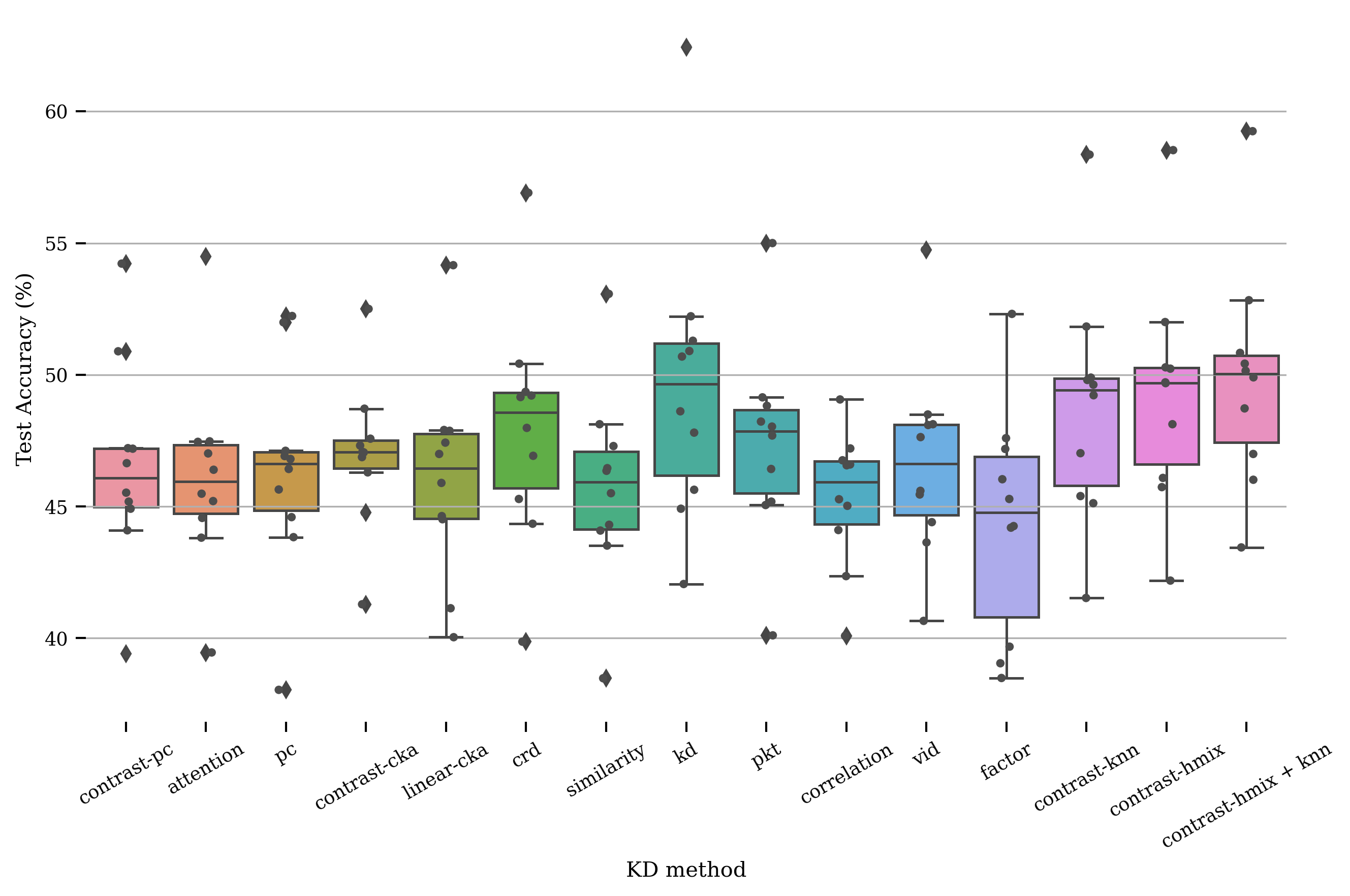}
    \caption{Tiny-ImageNet Boxplot Test Accuracy for Knowledge Distillation Approaches}
    \label{fig:tiny_imagenet_boxplot}
\end{figure}

\autoref{fig:cifar100_results}  and \autoref{fig:tiny_imagenet_results} (below the references section) show the KD results with the unscaled (i.e no [0, 1] normalization) color codings and \autoref{fig:cifar100_norm_results} and \autoref{fig:tiny_imagenet_norm_results} shows the corresponding [0, 1] row-normalized results to highlight the relative differences between each KD method. We note the last `Average Score' row displays the average performance over all student-teacher architecture pairs for each KD method.

\begin{figure*}[ht]
\centering
\begin{minipage}{1.\textwidth}
    \centering
    \includegraphics[scale=0.5]{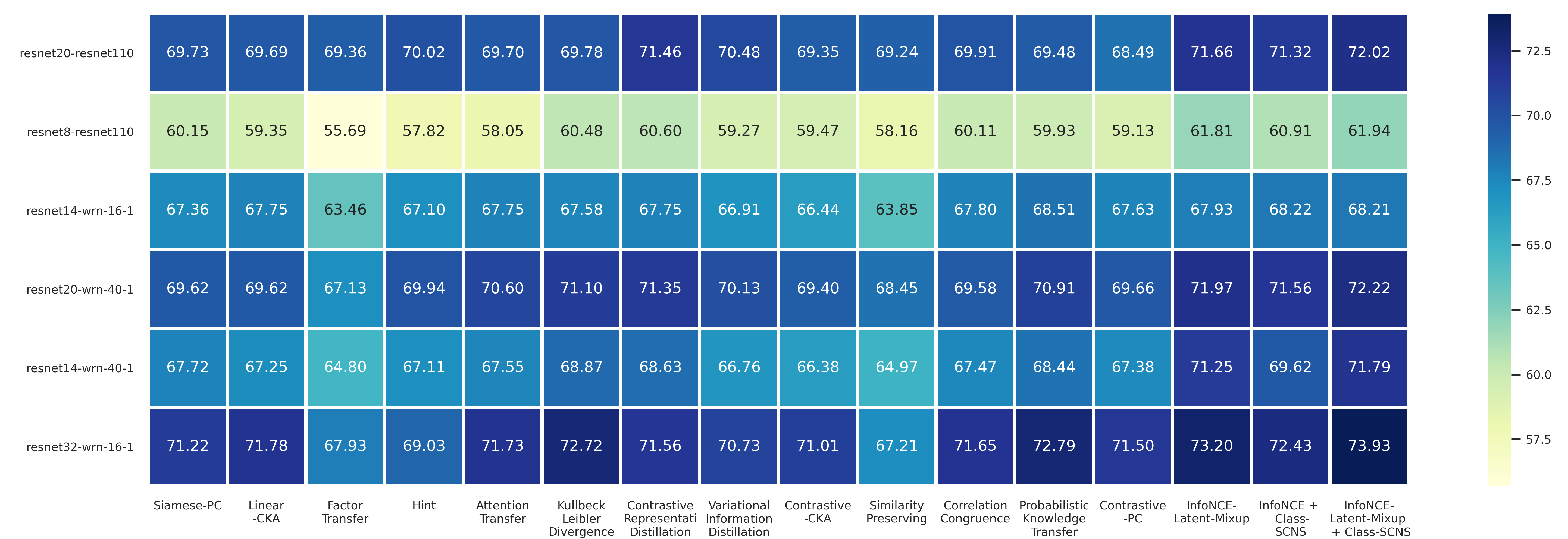}
    \caption{CIFAR-100 Test Accuracy for KD Approaches}
    \label{fig:cifar100_results}
\end{minipage}
\begin{minipage}{1.\textwidth}
    \centering
    \includegraphics[scale=0.54]{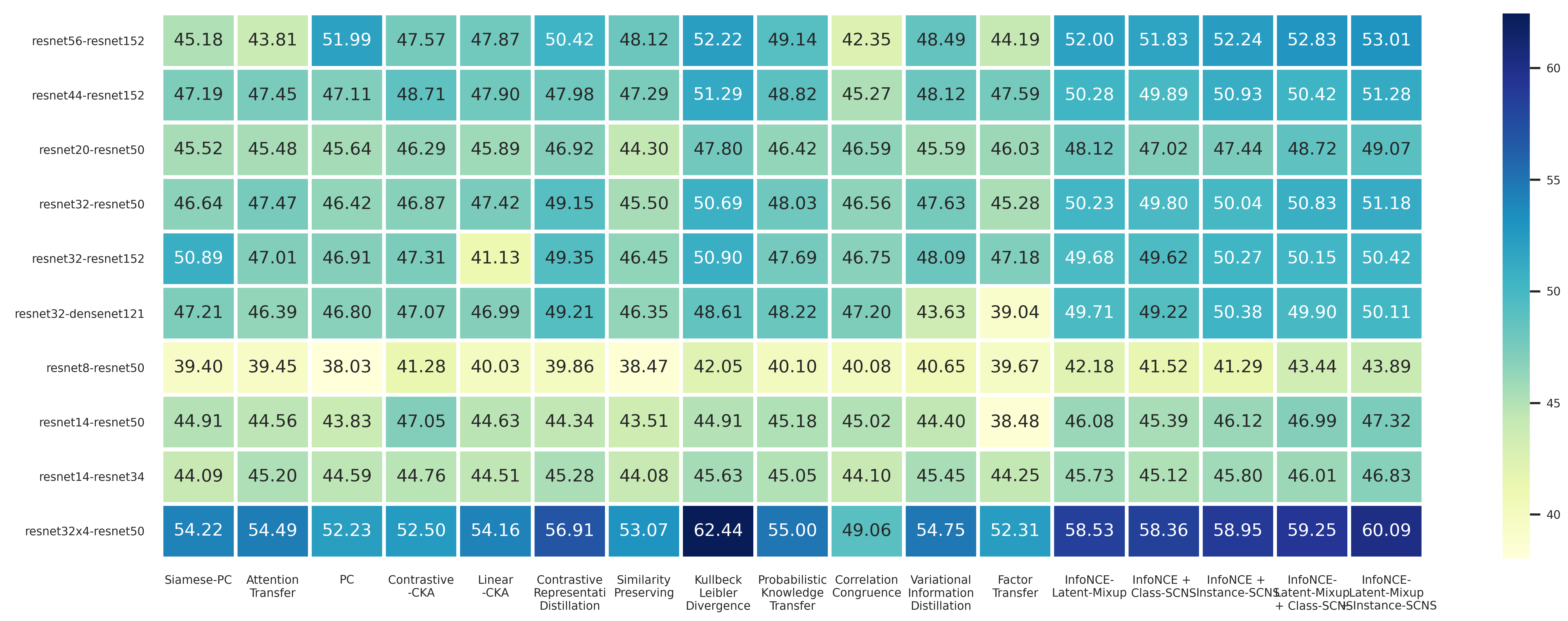}
    \caption{Tiny-Imagenet 200 Test Accuracy for KD Approaches}
    \label{fig:tiny_imagenet_results}
\end{minipage}
\end{figure*}

\begin{figure*}[ht]
\centering
\begin{minipage}{1.\textwidth}
    \centering
    \includegraphics[scale=0.55]{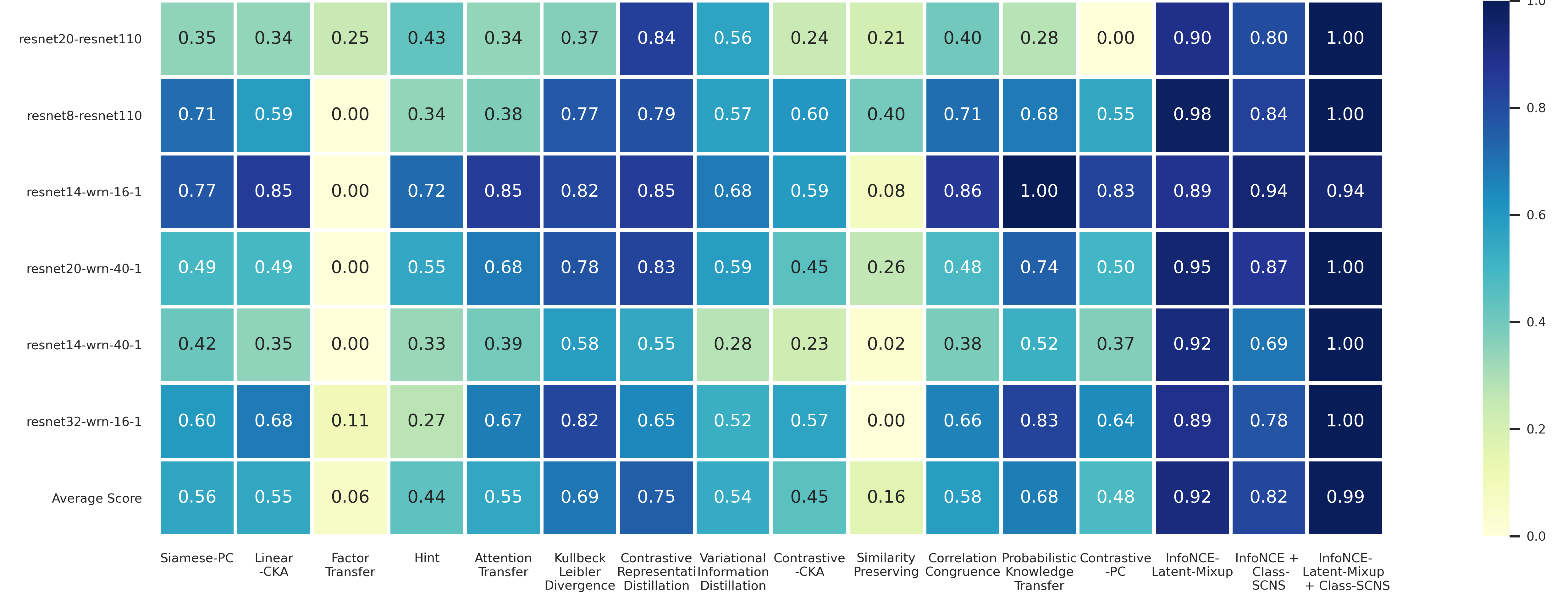}
    \caption{CIFAR-100 Normalized Test Accuracy for Knowledge Distillation Approaches}
    \label{fig:cifar100_norm_results}
\end{minipage}
\begin{minipage}{1.\textwidth}
    \centering
    \includegraphics[scale=0.47]{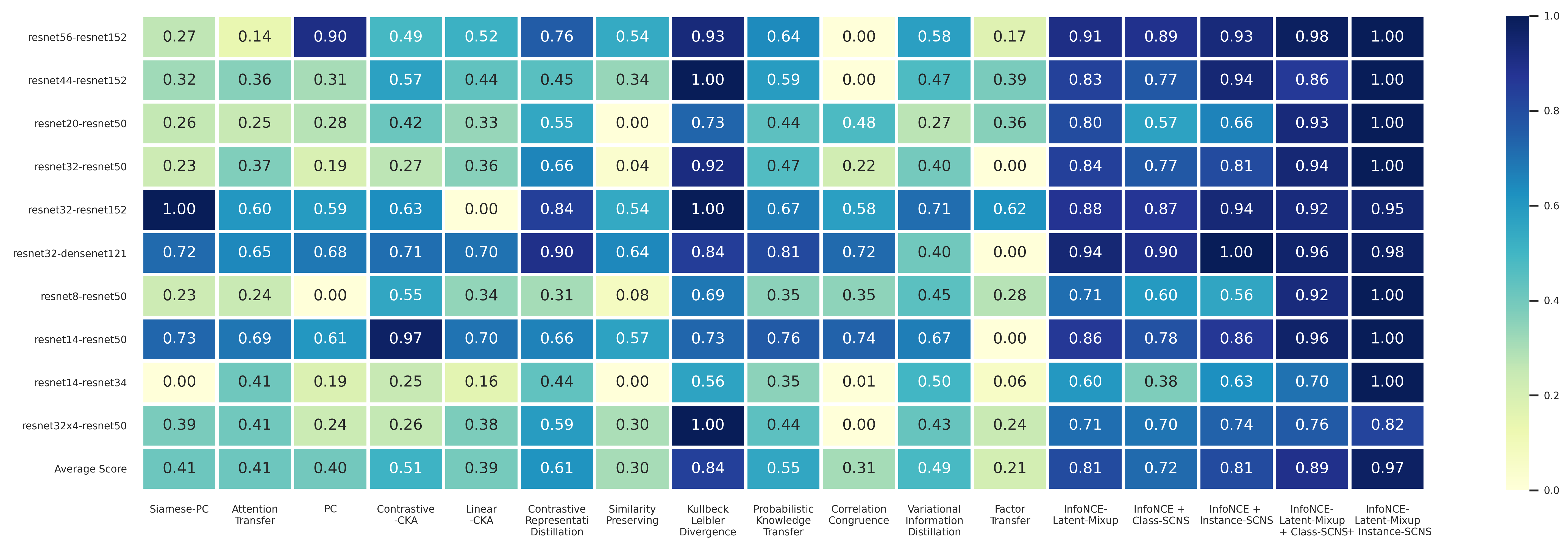}
    \caption{Tiny-Imagenet 200 Test Accuracy for Knowledge Distillation Approaches}
    \label{fig:tiny_imagenet_norm_results}
\end{minipage}
\end{figure*}

\Cref{fig:ti_corr_plot_1,fig:ti_corr_plot_2,fig:ti_corr_plot_3,fig:ti_corr_plot_4,fig:ti_corr_plot_5,fig:ti_corr_plot_6} shows the Tiny-ImageNet-200 embedding similarity of classes and \Cref{fig:corr_plot_1,fig:corr_plot_2,fig:corr_plot_3,fig:corr_plot_4,fig:corr_plot_5} shows the embedding similarity for CIFAR-100. 

\begin{figure*}[ht]
\centering    
\subfigure[Correlation Plot 1]{\label{fig:ti_corr_plot_1}\includegraphics[width=90mm]{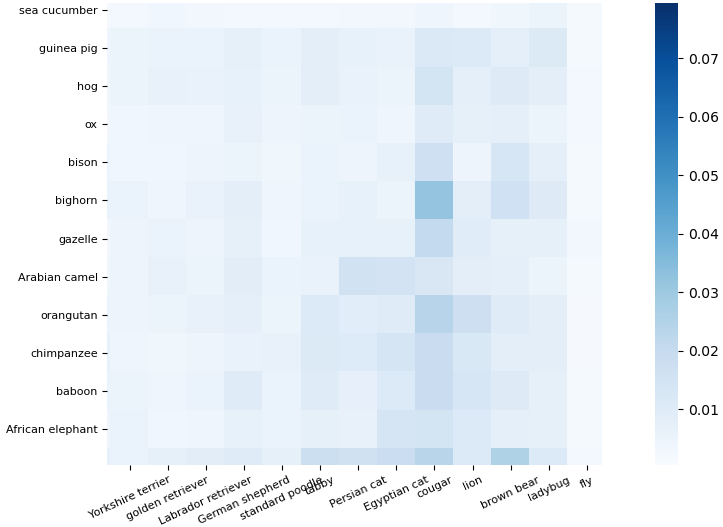}}%
\subfigure[Correlation Plot 2]{\label{fig:ti_corr_plot_2}\includegraphics[width=80mm]{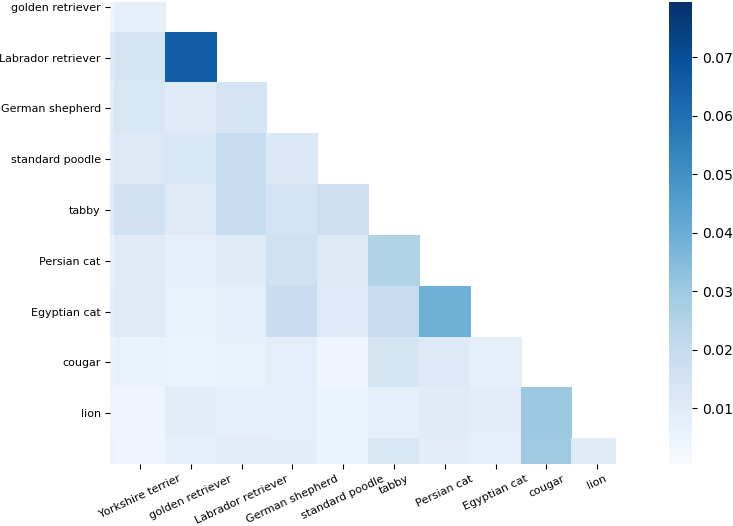}}
\subfigure[Correlation Plot 3]{\label{fig:ti_corr_plot_3}\includegraphics[width=90mm]{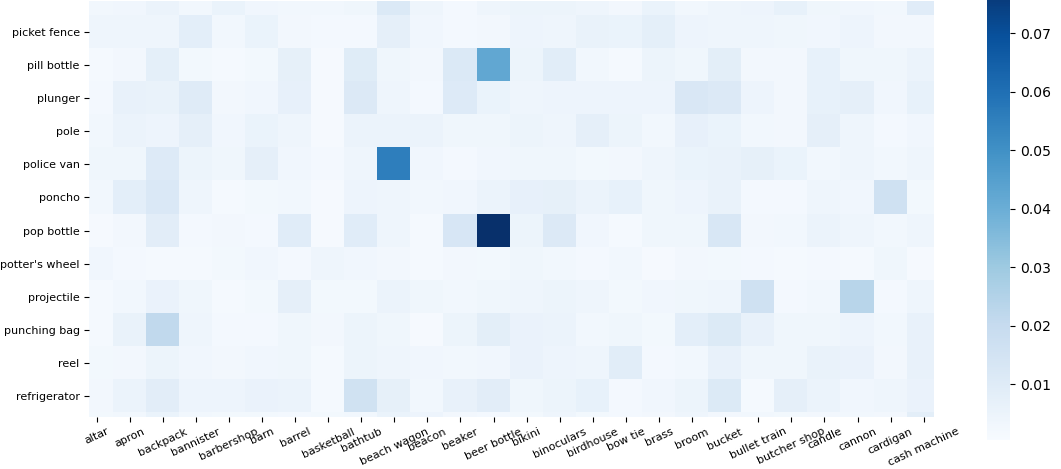}}%
\subfigure[Correlation Plot 4]{\label{fig:ti_corr_plot_4}\includegraphics[width=70mm]{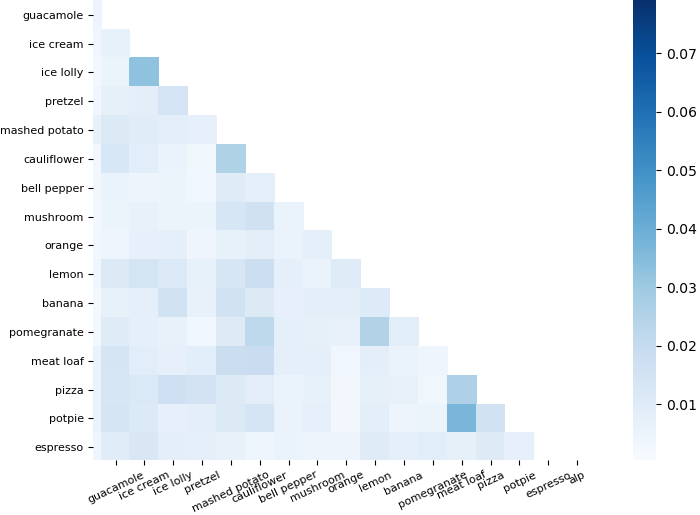}}
\subfigure[Correlation Plot 5]{\label{fig:ti_corr_plot_5}\includegraphics[width=90mm]{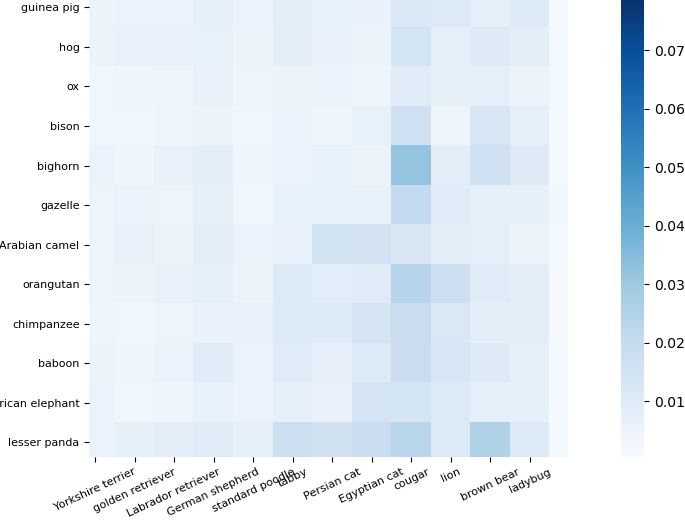}}%
\subfigure[Correlation Plot 6]{\label{fig:ti_corr_plot_6}\includegraphics[width=70mm]{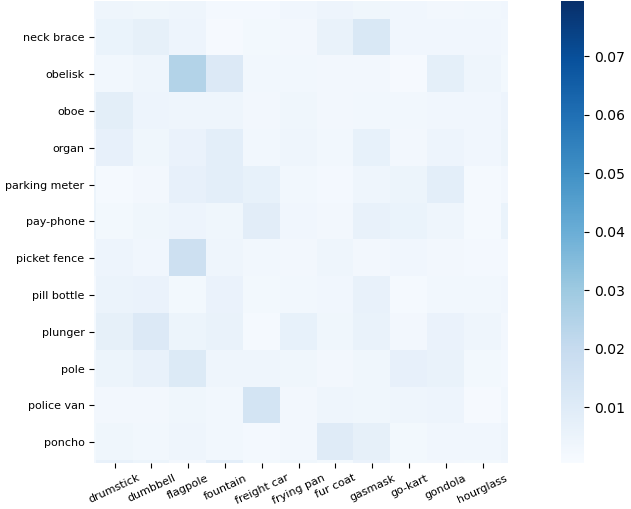}}
\end{figure*}

\begin{figure}[ht]
\centering    
\subfigure[Correlation Plot 1]{\label{fig:corr_plot_1}\includegraphics[width=80mm]{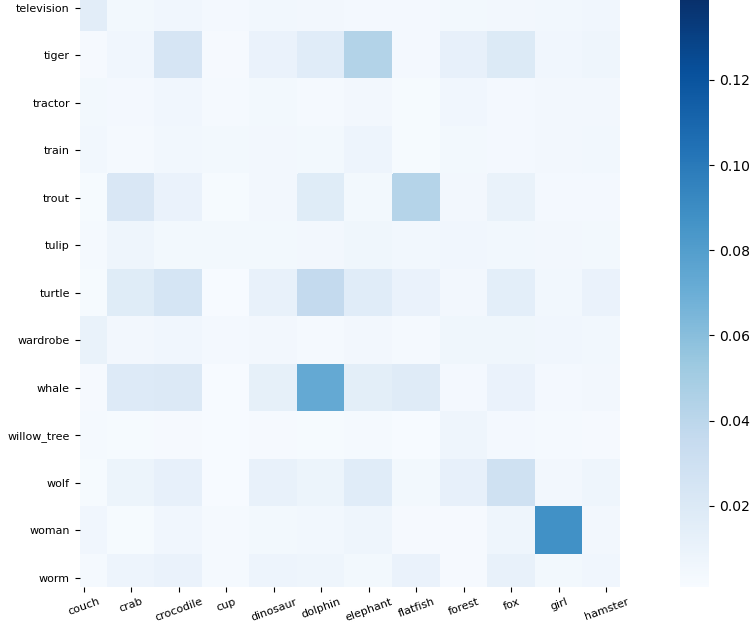}}%
\subfigure[Correlation Plot 2]{\label{fig:corr_plot_2}\includegraphics[width=85mm]{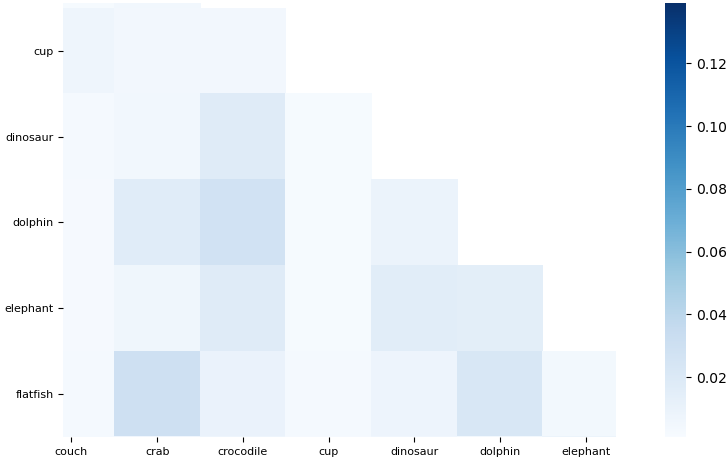}}
\subfigure[Correlation Plot 3]{\label{fig:corr_plot_3}\includegraphics[width=80mm]{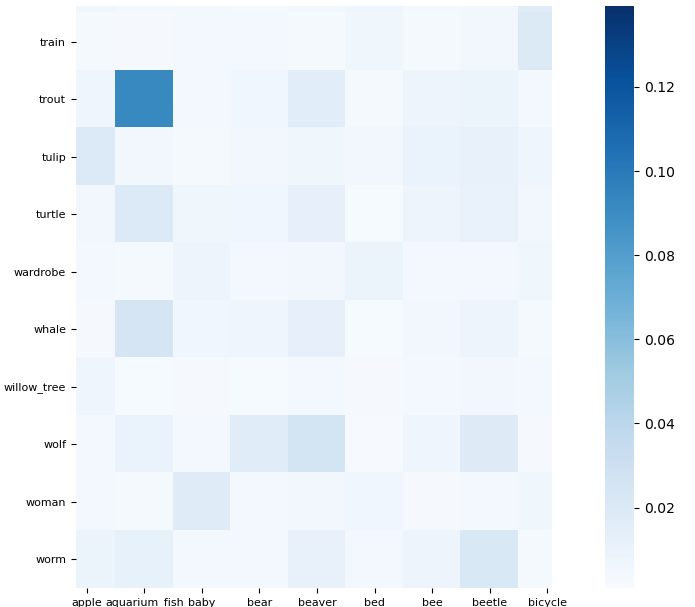}}%
\subfigure[Correlation Plot 4]{\label{fig:corr_plot_4}\includegraphics[width=75mm]{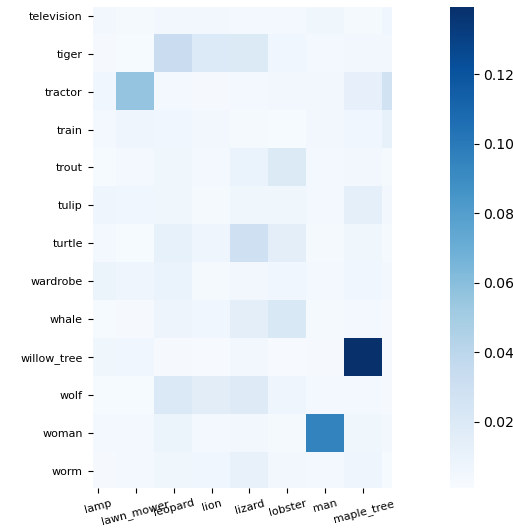}}
\subfigure[Correlation Plot 4]{\label{fig:corr_plot_5}\includegraphics[width=180mm]{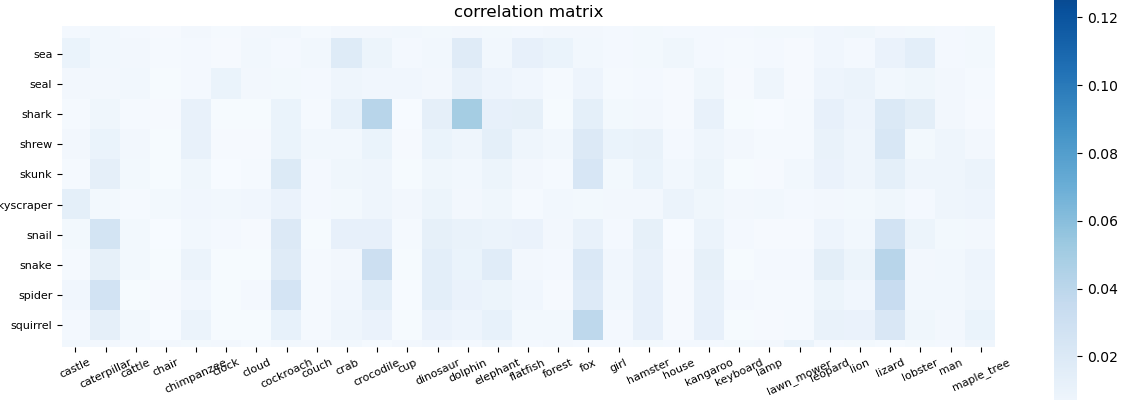}}
\end{figure}

\end{document}